\documentclass{article}

\usepackage[final]{corl_2019} 

\usepackage[utf8]{inputenc} 
\usepackage[T1]{fontenc}    
\usepackage{hyperref}       
\usepackage{url}            
\usepackage{booktabs}       
\usepackage{amsfonts}       
\usepackage{nicefrac}       
\usepackage{microtype}      
\usepackage{graphicx}
\usepackage{subfigure}
\usepackage{algorithm}
\usepackage{algorithmic}
\usepackage{amsmath}
\usepackage{amsfonts}
\usepackage{amssymb}
\usepackage{amsthm}
\usepackage{color}
\newtheorem{theorem}{Theorem}
\newtheorem{corollary}{Corollary}

\newtheorem{proposition}{Proposition}

\title{Better-than-Demonstrator Imitation Learning via Automatically-Ranked Demonstrations}

\author{
  Daniel S. Brown{\rm,} Wonjoon Goo {\rm and} Scott Niekum\\
  Department of Computer Science\\
  University of Texas at Austin 
  United States\\
  \texttt{\{dsbrown, wonjoon, sniekum\}@cs.utexas.edu} \\
}

\begin{document}
\maketitle


\begin{abstract}
The performance of imitation learning is typically upper-bounded by the performance of the demonstrator. While recent empirical results demonstrate that ranked demonstrations allow for better-than-demonstrator performance, preferences over demonstrations may be difficult to obtain, and little is known theoretically about when such methods can be expected to successfully extrapolate beyond the performance of the demonstrator. To address these issues, we first contribute a sufficient condition for better-than-demonstrator imitation learning and provide theoretical results showing why preferences over demonstrations can better reduce reward function ambiguity when performing inverse reinforcement learning. Building on this theory, we introduce Disturbance-based Reward Extrapolation (D-REX), a ranking-based imitation learning method that injects noise into a policy learned through behavioral cloning to automatically generate ranked demonstrations. These ranked demonstrations are used to efficiently learn a reward function that can then be optimized using reinforcement learning. We empirically validate our approach on simulated robot and Atari imitation learning benchmarks and show that D-REX outperforms standard imitation learning approaches and can significantly surpass the performance of the demonstrator. D-REX is the first imitation learning approach to achieve significant extrapolation beyond the demonstrator's performance without additional side-information or supervision, such as rewards or human preferences. By generating rankings automatically, we show that preference-based inverse reinforcement learning can be applied in traditional imitation learning settings where only unlabeled demonstrations are available. 
\end{abstract}



\section{Introduction} 
Imitation learning is a popular paradigm to teach robots and other autonomous agents to perform complex tasks simply by showing examples of how to perform the task. However, imitation learning methods typically find policies whose performance is upper-bounded by the performance of the demonstrator. While it is possible to learn policies that perform better than a demonstrator, existing methods either require access to a hand-crafted reward function \cite{taylor2011integrating,hester2018deep,gao2018reinforcement,sarafian2018safe} or a human supervisor who acts as a reward or value function during policy learning \cite{warnell2017deep,christiano2017deep,browngoo2019trex}. Recent empirical results \cite{browngoo2019trex} give evidence that better-than-demonstrator performance can be achieved using ranked demonstrations; however, theoretical conditions for improvement over a demonstrator are lacking. This lack of theory makes it difficult to predict when current imitation learning approaches may exceed the performance of the demonstrator and precludes using theory to design better imitation learning algorithms.

In this paper, we first present theoretical results for when better-than-demonstrator performance is possible in an inverse reinforcement learning (IRL) setting \cite{abbeel2004apprenticeship}, where the goal is to recover a reward function from demonstrations. 
We then present theoretical results demonstrating that rankings (or alternatively, pairwise preferences) over demonstrations can enable better-than-demonstrator performance by reducing error and ambiguity in the learned reward function. Next, we address the problem of leveraging the benefits of reward learning via ranked demonstrations in a way that does not require human rankings.
Recently, Brown et al. \cite{browngoo2019trex} proposed Trajectory-ranked Reward Extrapolation (T-REX), an imitation learning approach that uses a set of ranked demonstrations to learn a reward function that allows better-than-demonstrator performance without requiring human supervision during policy learning. However, requiring a demonstrator to rank demonstrations can be tedious and error prone, and precludes learning from prerecorded, unranked demonstrations, or learning from demonstrations of similar quality that are difficult to rank. Thus, we investigate whether it is possible to generate a set of ranked demonstrations, in order to surpass the performance of a demonstrator, without requiring supervised preference labels or reward information.

\begin{figure}
    \centering
    \subfigure[Demonstration]{
        \includegraphics[width=.18\linewidth]{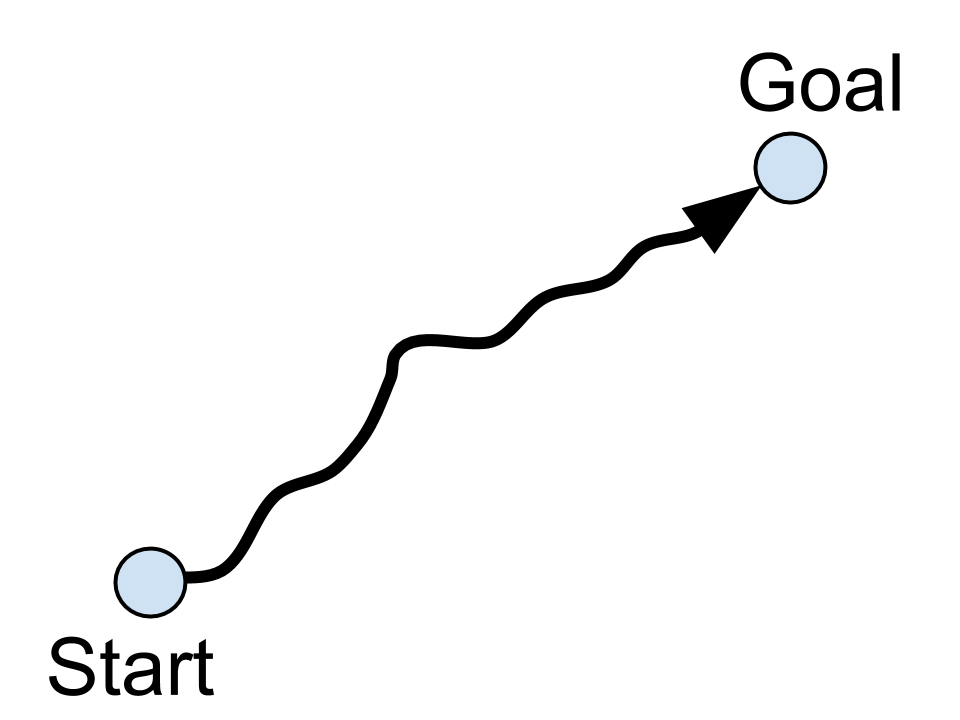}
        
    }
     \subfigure[Small noise]{
        \includegraphics[width=.18\linewidth]{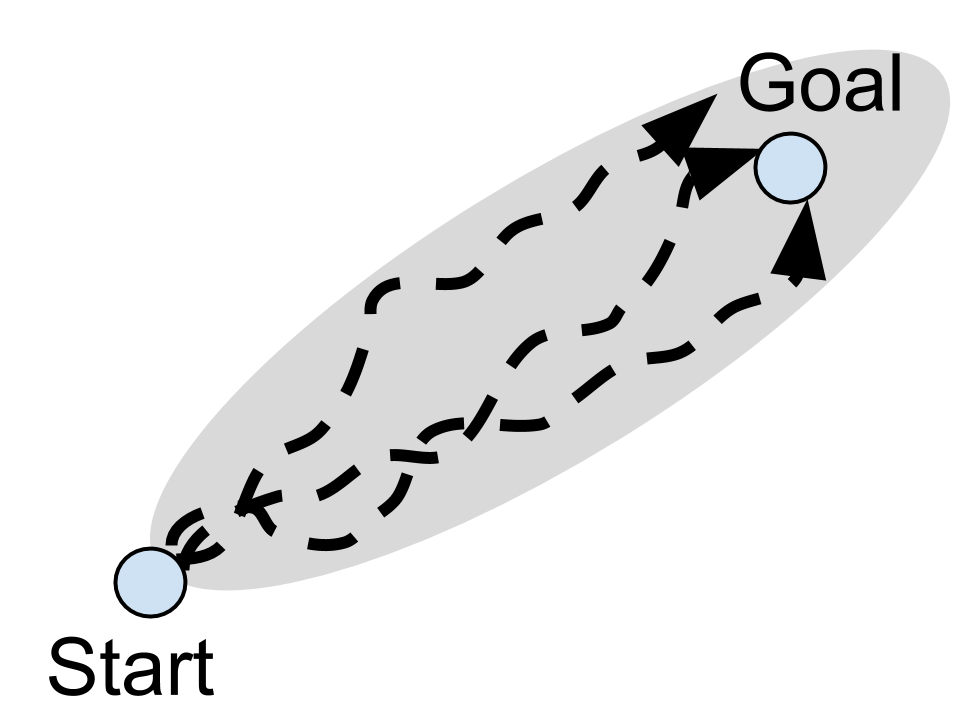}
        
    }
    \subfigure[Larger noise]{
        \includegraphics[width=.18\linewidth]{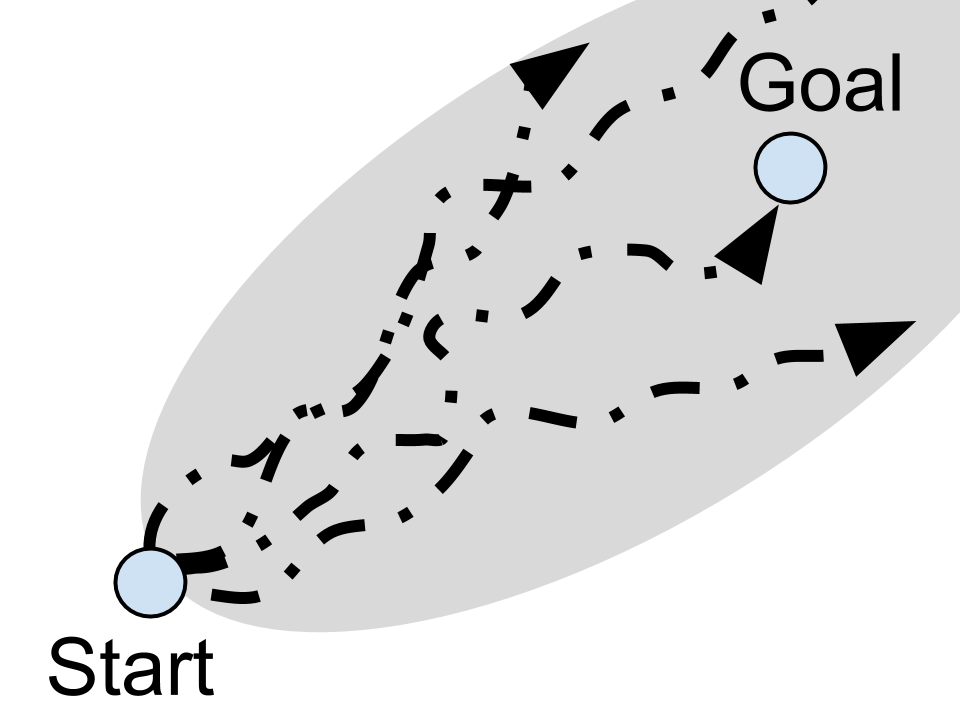}
        
    }
     \subfigure[Learned reward function from ranking: $(a) \succ (b) \succ (c)$]{
        \includegraphics[width=.18\linewidth]{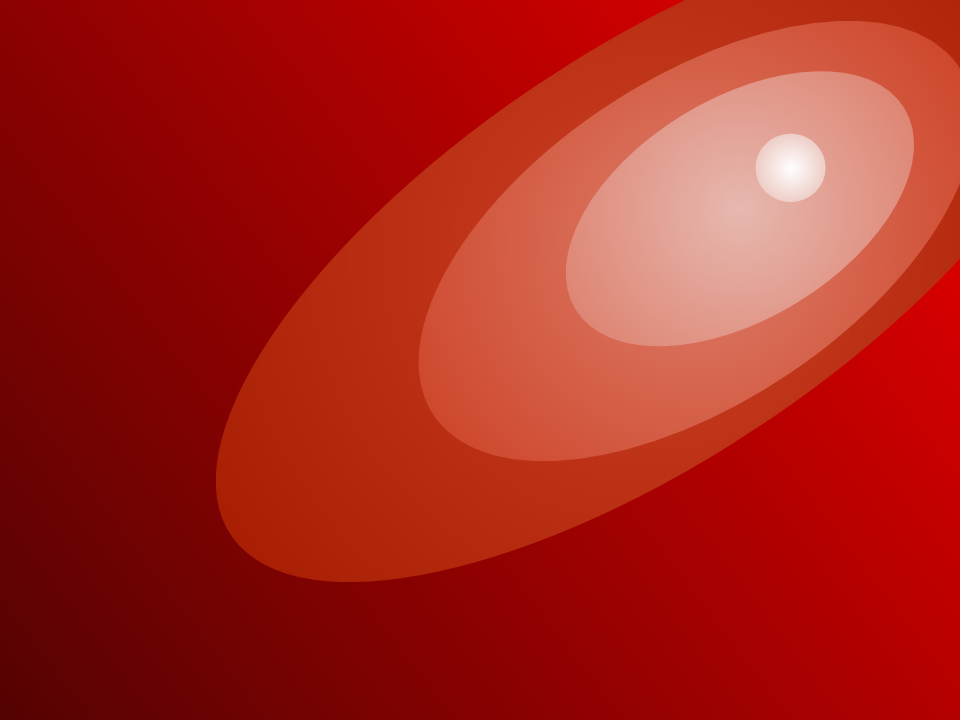}
       
    }
     \subfigure[Optimized policy]{
        \includegraphics[width=.18\linewidth]{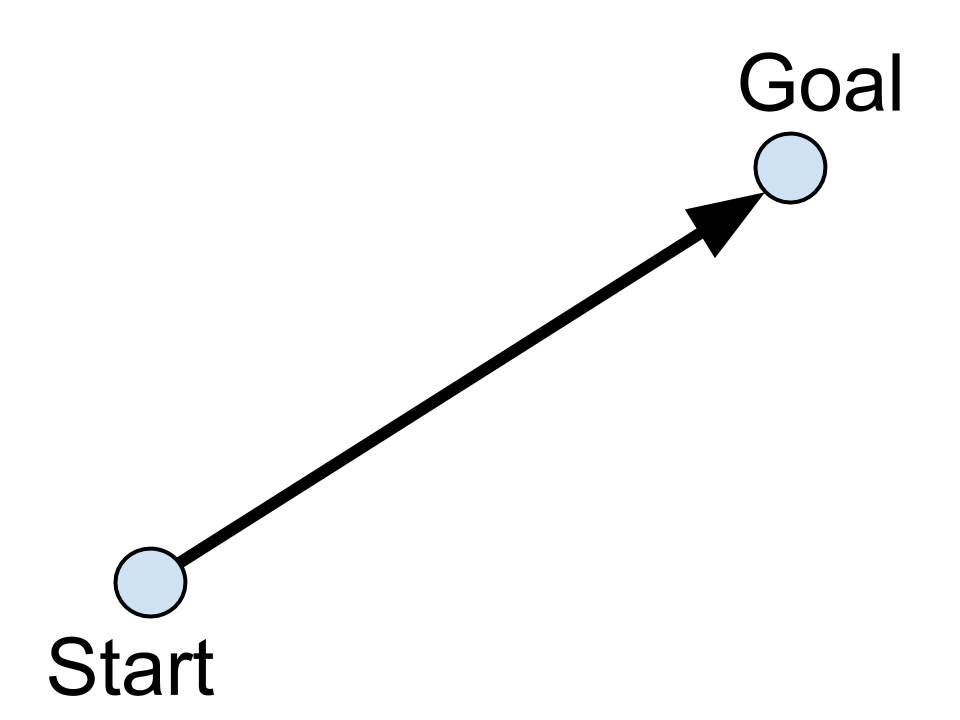}
        
    }
    \caption{\textbf{D-REX high-level approach:} given a suboptimal demonstration (a), we run behavioral cloning to approximate the demonstrator's policy. By progressively adding more noise to this cloned policy ((b) and (c)), we are able to automatically synthesize a preference ranking: $(a) \succ (b) \succ (c)$. Using this ranking, we learn a reward function (d) which is then optimized using reinforcement learning to obtain a policy (e) that performs better than the demonstrator.}
    \label{fig:addingNoiseForRankings}
\end{figure}

We propose Disturbance-based Reward Extrapolation (D-REX), a ranking-based reward learning algorithm that does not require ranked demonstrations. Our approach injects noise into a policy learned through behavioral cloning to automatically generate ranked policies of varying performance. D-REX makes the weak assumption that the demonstrations are better than a purely random policy, and that adding increasing levels of noise into a cloned policy will result in increasingly worse performance, converging to a random policy in the limit. Our approach is summarized in Figure~\ref{fig:addingNoiseForRankings}. The intuition behind this approach is that generating ranked trajectories via noise injection reveals relative weightings between reward features: features that are more prevalent in noisier trajectories are likely inversely related to the reward, whereas features that are more common in noise-free trajectories are likely features which are positively correlated with the true reward. Furthermore, adding noise provides a form of feature selection since, if a feature is equally common across all levels of noise, then it likely has no impact on the true reward function and can be ignored. 

By automatically generating rankings, preference-based imitation learning methods \cite{sadigh2017active,ibarz2018reward,palan2019learning,browngoo2019trex} can be applied in standard imitation learning domains where rankings are unavailable. We demonstrate this by combining automatic rankings via noise-injections with a state-of-the-art imitation learning algorithm that uses ranked demonstrations \cite{browngoo2019trex}. We empirically validate our approach on simulated robotics and Atari benchmarks and find that D-REX results in policies that can both significantly outperform the demonstrator as well as significantly outperform standard imitation learning. To the best of our knowledge, D-REX is the first imitation learning approach to achieve significant performance improvements over the demonstrations without requiring extra supervision or additional side-information, such as ground-truth rewards or human preferences.

\section{Related Work}
Imitation learning has grown increasingly popular in recent years \cite{osa2018algorithmic,arora2018survey}, with many applications in robotics \cite{finn2016guided,mandlekar2018roboturk,kroemer2019review,codevilla2018end}. However, little work has addressed the problem of achieving better-than-demonstrator performance. 
When ground-truth rewards are known, it is common to initialize a policy using demonstrations and then improve this policy using reinforcement learning \cite{taylor2011integrating,pastor2011skill,hester2018deep}. However, 
designing good reward functions for reinforcement learning can be difficult and can easily lead to unintended behaviors \cite{ng1999policy,amodei2016concrete}.

Inverse reinforcement learning can be used to estimate a demonstrator's intent by inferring a reward function that explains the demonstrations. While there has been some work on learning from suboptimal demonstrations, most approaches either require suboptimal demonstrations to be manually clustered \cite{grollman2011donut} or labeled \cite{shiarlis2016inverse}. Other methods are robust to unlabeled, suboptimal demonstrations, but require most demonstrations to come from an expert in order to correctly identify which demonstrations are anomalous \cite{zheng2014robust,choi2019robust}. Syed and Schapire \cite{syed2008game} proved that knowing which features contribute positively or negatively to the true reward allows an apprenticeship policy to outperform the demonstrator. However, their approach requires hand-crafted, linear features, knowledge of the true signs of the rewards features, and repeatedly solving a Markov decision process. 

Preference learning \cite{eric2008active} is another way to potentially learn better-than-demonstrator policies. \citet{sadigh2017active} and \citet{christiano2017deep} propose reward learning approaches that use active learning to collect pairwise preferences labels. \citet{ibarz2018reward} and \citet{palan2019learning} combine demonstrations with active preference learning during policy optimization. Rather than collecting pairwise preferences via active queries, other research has investigated learning better-than-demonstrator policies via prelabeled demonstrations \cite{wu2019imitation,browngoo2019trex}. Brown et al. \cite{browngoo2019trex} propose Trajectory-ranked Reward Extrapolation (T-REX), an algorithm that uses a set of pre-ranked demonstrations to learn a reward function for policy optimization. Brown et al. evaluate T-REX on a variety of MuJoCo and Atari benchmarks and show that policies optimized via T-REX are able to consistently surpass the performance of a suboptimal demonstrator, but provide no theory to shed light on when extrapolation is possible. Our work provides a better theoretical understanding of when better-than-demonstrator performance is possible and why ranked demonstrations can help. Furthermore, our work demonstrates for the first time that ranking-based imitation learning approaches are applicable even in cases where human rankings are unavailable.

Prior work on imitation learning has investigated the use of random or noisy trajectories. Boularias et al. \cite{boularias2011relative} and Kalakrishnan et al. \cite{kalakrishnan2013learning} use uniformly random and locally perturbed trajectories, respectively, to estimate the partition function for Maximum Entropy inverse reinforcement learning \cite{ziebart2008maximum}. Both methods seek a linear combination of predefined features such that the returns of the demonstrations are maximized with respect to the random trajectories. These methods can be seen as a special case of our proposed method, where only one level of noise is used and where the reward function is represented as a linear combination of known features.  
Disturbances for Augmenting Robot Trajectories (DART) \cite{laskey2017dart} is a recently proposed behavioral cloning approach 
that adds noise during demonstrations to collect a richer set of state-action pairs for behavioral cloning. DART avoids the problem of compounding error that is common to most behavioral cloning approaches by repeatedly requesting and perturbing new demonstrations.
Instead of repeatedly collecting perturbed trajectories from the demonstrator, we instead propose to collect a small number of initial demonstrations, run behavioral cloning once, and then inject varying amounts of noise into the cloned policy. This automatically creates a large set of ranked demonstrations for reward learning, without requiring a human to provide preference labels.

\section{Problem Statement}
Our goal is to achieve better-than-demonstrator performance via imitation learning. We model the environment as a Markov decision process (MDP) consisting of a set of states $\mathcal{S}$, actions $\mathcal{A}$, transition probabilities $P:\mathcal{S}\times \mathcal{A} \times \mathcal{S} \rightarrow [0,1]$, reward function $R^*: \mathcal{S} \rightarrow \mathbb{R}$, and discount factor $\gamma \in [0,1)$. 
A policy $\pi$ is a probability distribution over actions given state. 
Given a policy and an MDP, the expected discounted return of the policy is given by $J(\pi|R^*) = \mathbb{E}_\pi[\sum_{t=0}^\infty \gamma^t R^*(s_t)]$. Similarly, the return of a trajectory consisting of states and actions, $\tau=(s_0, a_0, s_1, a_1,\ldots,s_T, a_T)$, is given by $J(\tau|R^*) = \sum_{t=0}^T \gamma^t R^*(s_t)$.  

We assume that we have no access to the true reward function of the MDP. Instead, we are given a set of $m$ demonstrations $\mathcal{D} = \{\tau_1, \ldots \tau_m\}$, where each demonstrated trajectory is a sequence of states and actions, $\tau_i = (s_0,a_0,s_1,a_1, \ldots)$. We assume that the demonstrator is attempting (possibly unsuccessfully) to follow a policy that optimizes the true reward function $R^*$. Given the demonstrations $\mathcal{D}$, we wish to find a policy $\hat{\pi}$ that can extrapolate beyond the performance of the demonstrator. We say a policy $\hat{\pi}$ can extrapolate beyond of the performance of the demonstrator if it achieves a larger expected return than the demonstrations, when evaluated under the true reward function $R^*$, i.e., $J(\hat{\pi}|R^*) > J(\mathcal{D}|R^*)$, where $J(\mathcal{D}|R^*) = \frac{1}{|\mathcal{D}|}\sum_{\tau \in \mathcal{D}} J(\tau|R^*)$. Similarly, we say that a learned policy $\hat{\pi}$ extrapolates beyond the performance of the best demonstration if $J(\hat{\pi}|R^*) > \max_{\tau \in \mathcal{D}} J(\tau|R^*)$.

\section{Extrapolating Beyond a Demonstrator: Theory}
We first provide a sufficient condition under which it is possible to achieve better-than-demonstrator performance in an inverse reinforcement learning (IRL) setting, where the goal is to recover the demonstrator's reward function which is then used to optimize a policy \cite{arora2018survey}.
We consider a learner that approximates the reward function of the demonstrator with a linear combination of features: $R(s)  = w^T \phi(s)$.\footnote{Our results also hold for reward functions of the form $R(s,a) = w^T \phi(s,a)$. } These can be arbitrarily complex features, such as the activations of a deep neural network. The expected return of a policy when evaluated on $R(s)$ is given by 
\begin{equation}
J(\pi|R) = \mathbb{E}_\pi\bigg[\sum_{t=0}^\infty \gamma^t R(s_t)\bigg] = w^T\mathbb{E}_\pi\bigg[\sum_{t=0}^\infty \gamma^t \phi(s_t)\bigg] = w^T \Phi_\pi,
\end{equation}
where $\Phi_\pi$ are the expected discounted feature counts that result from following the policy $\pi$.


\begin{theorem}
If the estimated reward function is $\hat{R}(s) = w^T \phi(s)$, the true reward function is $R^*(s) = \hat{R}(s) + \epsilon(s)$ for some error function $\epsilon : \mathcal{S} \rightarrow \mathbb{R}$, and $\|w\|_1 \leq 1$, then extrapolation beyond the demonstrator, i.e., $J(\hat{\pi}|R^*) > J(\mathcal{D}| R^*)$, is guaranteed if : 
\begin{equation}
J(\pi^*_{R^*}|R^*) - J(\mathcal{D}|R^*) > \epsilon_\Phi + \frac{2 \|\epsilon\|_\infty}{1 - \gamma}
\end{equation}
where $\pi^*_{R^*}$ is the optimal policy under $R^*$, $\epsilon_\Phi = \|\Phi_{\pi^*_{R^*}} - \Phi_{\hat{\pi}}\|_\infty$ and $\|\epsilon\|_\infty=\sup\left\{\,\left|\epsilon(s)\right|:s\in \mathcal{S}\,\right\}$.
\end{theorem}
All proofs are given in the appendix.
Intuitively, extrapolation depends on the demonstrator being sufficiently suboptimal, the error in the learned reward function being sufficiently small, and the state occupancy of the imitation policy, $\hat{\pi}$, being sufficiently close to $\pi^*_{R^*}$. If we can perfectly recover the reward function, then reinforcement learning can be used to ensure that $\epsilon_\Phi$ is small. Thus, we focus on improving the accuracy of the learned reward function via automatically-ranked demonstrations. The learned reward function can then be optimized with any reinforcement learning algorithm \cite{sutton1998introduction}.

\subsection{Extrapolation via ranked demonstrations}
The previous results demonstrate that in order to extrapolate beyond a suboptimal demonstrator, it is sufficient to have small reward approximation error and a good policy optimization algorithm. However, the following proposition, adapted from \cite{pabloranked}, shows that the reward function learned by standard IRL may be quite superficial and miss potentially important details, whereas enforcing a ranking over trajectories leads to a more accurate estimate of the true reward function. 

\begin{proposition}\label{prop:mdp_example_woutproof}
There exist MDPs with true reward function $R^*$, expert policy $\pi_E$, approximate reward function $\hat{R}$, and non-expert policies $\pi_1$ and $\pi_2$, such that 
\begin{eqnarray}
&&\pi_E = \arg \max_{\pi \in \Pi} J(\pi |R^*) \;\text{  and  }\; J(\pi_1 | R^*) \ll J(\pi_2 | R^*) \label{eqn:mdpexample1} \\
&&\pi_E = \arg \max_{\pi \in \Pi} J(\pi | \hat{R}) \;\text{  and  }\; J(\pi_1 | \hat{R}) = J(\pi_2 | \hat{R}) \label{eqn:mdpexample2}. 
\end{eqnarray}
However, enforcing a preference ranking over trajectories, $\tau^* \succ \tau_2 \succ \tau_1$, where $\tau^* \sim \pi^*$, $\tau_2 \sim \pi_2$, and $\tau_1 \sim \pi_1$, results in a learned reward function $\hat{R}$, such that 
\begin{equation}
\pi_E = \arg \max_{\pi \in \Pi} J(\pi |\hat{R}) \;\text{  and  }\; J(\pi_1 | \hat{R}) <  J(\pi_2 | \hat{R}) \label{eqn:mdpexample3}.
\end{equation}
\end{proposition}

Proposition~\ref{prop:mdp_example_woutproof} proves the existence of MDPs where an approximation of the true reward leads to an optimal policy, yet the learned reward reveals little about the underlying reward structure of the MDP.
This is problematic for several reasons. The first problem is that if the learned reward function is drastically different than the true reward, this can lead to poor generalization. Another problem is that many learning from demonstration methods are motivated by  providing \textit{non-experts} the ability to program by example. Some non-experts will be good at personally performing a task, but may struggle when giving kinesthetic demonstrations \cite{akgun2012trajectories} or teleoperating a robot \cite{chuck2017statistical,kent2017comparison}. 
Other non-experts may not be able to personally perform a task at a high level of performance due to lack of precision or timing, or due to physical limitations or impairment. Thus, the standard IRL approach of finding a reward function that maximizes the likelihood of the demonstrations may lead to an incorrect, superficial reward function that overfits to suboptimal user behavior in the demonstrations. 

Indeed, it has been proven that it is impossible to recover the correct reward function without additional information beyond observations, regardless of whether the policy is optimal \cite{ng2000algorithms} or suboptimal  \cite{armstrong2018occam}. As demonstrated in Proposition~\ref{prop:mdp_example_woutproof}, preference rankings can help to alleviate reward function ambiguity. If the true reward function is a linear combination of features, then the feasible region of all reward functions that make a policy optimal can be defined as an intersection of half-planes \cite{brown2019machine}: 
$H_\pi = \bigcap_{\pi' \in \Pi} w^T (\Phi_\pi - \Phi_{\pi'}) \geq 0$.
We define the \textit{reward ambiguity}, $G(H_{\pi})$, as the volume of this intersection of half-planes: 
$G(H_{\pi}) = \text{Volume}(H_\pi)$,
where we assume without loss of generality that $\|w\| \leq 1$, to ensure this volume is bounded.
In Appendix~\ref{app:fullranking} we prove that a total ranking over policies can result in less reward ambiguity than performing IRL on the optimal policy.
\begin{proposition}
Given a policy class $\Pi$, an optimal policy $\pi^* \in \Pi$ and a total ranking over $\Pi$, and true reward function $R^*(s) = w^T \phi(s)$,  the reward ambiguity resulting from $\pi^*$ is greater than or equal to the reward ambiguity of using a total ranking, i.e., $G(H_\pi^*) \geq G(H_{\rm ranked})$. 
\end{proposition}

Learning a reward function that respects a set of strictly ranked demonstrations avoids some of the ill-posedness of IRL \cite{ng2000algorithms} by eliminating a constant, or all-zero reward function. Furthermore, ranked demonstrations provide explicit information about both what to do as well as what \textit{not} to do in an environment and each pairwise preference over trajectories gives a half-space constraint on feasible reward functions. 
In Appendix~\ref{app:halfspace} we prove that sampling random half-space constraints results in an exponential decrease in reward function ambiguity.
\begin{corollary}
To reduce reward function ambiguity by x\% it suffices to have $k = \log_2 (1/(1-x/100))$
random half-space constraints over reward function weights.
\end{corollary}
In practice, sampling random half-space constraints on the ground-truth reward function is infeasible. Instead, our proposed approach for better-than-demonstrator imitation learning uses noise injection to produce a wide variety of automatically-ranked of demonstrations in order to reduce the learner's reward function ambiguity. As we show in the next section, automatically-generating preferences over demonstrations also improves the efficiency of IRL by removing the need for an MDP solver in the inner-loop and turning IRL into a supervised learning problem \cite{browngoo2019trex}. \citet{amin2016towards} proved that a logarithmic number of demonstrations from a family of MDPs with different transition dynamics is sufficient to resolve reward ambiguity in IRL. We generate ranked trajectories via noise injection, which can be seen as an efficient heuristic for generating demonstrations under different transition dynamics.  

\section{Algorithm}
We now describe our approach for achieving better-than-demonstrator imitation learning without requiring human-provided preference labels. We first briefly review a recent state-of-the-art IRL algorithm that learns from ranked demonstrations. We then describe our proposed approach to generate these rankings automatically via noise injection. Videos and code are available online.\footnote{The project website and code can be found at \url{https://dsbrown1331.github.io/CoRL2019-DREX/}}

\subsection{Trajectory-ranked Reward Extrapolation (T-REX)}
\label{sec:trex}
Given a sequence of $m$ demonstrations ranked from worst to best, $\tau_1, \ldots, \tau_m$, T-REX \cite{browngoo2019trex} first performs reward inference by approximating the reward at state $s$ using a neural network, $\hat{R}_\theta(s)$, such that $\sum_{s \in \tau_i} \hat{R}_\theta(s) < \sum_{s \in \tau_j} \hat{R}_\theta (s)$ when $\tau_i \prec \tau_j$.
The reward function $\hat{R}_\theta$ is learned via supervised learning, using a pairwise ranking loss \cite{cao2007learning} based on the Luce-Shephard choice rule \cite{luce2012individual}:
\begin{equation}
 \mathcal{L}(\theta) \approx - \frac{1}{|\mathcal{P}|}\sum_{(i,j) \in \mathcal{P}} \log \frac{\exp \displaystyle\sum_{s \in \tau_j} \hat{R}_\theta(s)}{\exp \displaystyle\sum_{s \in \tau_i} \hat{R}_\theta(s) + \exp \displaystyle\sum_{s \in \tau_j} \hat{R}_\theta(s)},
\end{equation}
where $\mathcal{P} = \{(i,j) : \tau_i \prec \tau_j  \}$.
After learning a reward function, T-REX can be combined with any RL algorithm to optimize a policy $\hat{\pi}$ with respect to $\hat{R}_\theta(s)$. \citet{browngoo2019trex} demonstrated that T-REX typically results in policies that perform significantly better than the best demonstration.

\subsection{Disturbance-based Reward Extrapolation (D-REX)}

 \begin{algorithm}[t] 
 \caption{D-REX: Disturbance-based Reward Extrapolation} 
 \label{alg:D-REX} 
 \begin{algorithmic}[1] 
     \REQUIRE Demonstrations $\mathcal{D}$, noise schedule $\mathcal{E}$, number of rollouts $K$
     \STATE Run behavioral cloning on demonstrations $\mathcal{D}$ to obtain policy $\pi_{\rm BC}$
 	\FOR{$\epsilon_i \in \mathcal{E}$}
 		\STATE Generate a set of $K$ trajectories from a noise injected policy $\pi_{\rm BC}(\cdot|\epsilon_i)$
 	\ENDFOR
 	\STATE Generate automatic preference labels $\tau_i \prec \tau_j$ if $\tau_i\sim \pi_{\rm BC}(\cdot | \epsilon_i)$, $\tau_j\sim \pi_{\rm BC}(\cdot | \epsilon_j)$, and $\epsilon_i > \epsilon_j$
 	\STATE Run T-REX \cite{browngoo2019trex} on automatically ranked trajectories to obtain $\hat{R}$
 	\STATE Optimize policy $\hat{\pi}$ using reinforcement learning with reward function $\hat{R}$
 	\RETURN $\hat{\pi}$
 \end{algorithmic}
 \end{algorithm}

We now describe Disturbance-based Reward Extrapolation (D-REX), our proposed approach for automatically generating ranked demonstrations. Our approach is summarized in Algorithm~\ref{alg:D-REX}. We first take a set of unranked demonstrations and use behavioral cloning to learn a policy $\pi_{\rm BC}$. Behavioral cloning \cite{bain1999framework} treats each state action pair $(s,a) \in \mathcal{D}$ as a training example and seeks a policy $\pi_{\rm BC}$ that maps from states to actions. We model $\pi_{\rm BC}$ using a neural network with parameters $\theta_{\rm BC}$ and find these parameters using maximum-likelihood estimation such that 
$\theta_{\rm BC} = \arg \max_\theta \prod_{(s,a) \in \mathcal{D}} \pi_{\rm BC}(a | s)$. 
By virtue of the optimization procedure, $\pi_{\rm BC}$ will usually only perform as well as the average performance of the demonstrator---at best it may perform slightly better than the demonstrator if the demonstrator makes mistakes approximately uniformly at random. 

Our main insight is that if $\pi_{\rm BC}$ is significantly better than the performance of a completely random policy, then we can inject noise into $\pi_{\rm BC}$ and interpolate between the performance of $\pi_{\rm BC}$ and the performance of a uniformly random policy. In Appendix~\ref{app:degredation}, we prove that given a noise schedule $\mathcal{E}= (\epsilon_1, \epsilon_2, \ldots, \epsilon_d)$ consisting of a sequence of noise levels such that $\epsilon_1 > \epsilon_2 > \ldots > \epsilon_d$, then with high-probability, $J(\pi_{\rm BC}(\cdot|\epsilon_1)) < J(\pi_{\rm BC}(\cdot|\epsilon_2)) < \cdots < J(\pi_{\rm BC}(\cdot |\epsilon_d))$.
Given noise level $\epsilon \in \mathcal{E}$, we inject noise via an $\epsilon$-greedy policy such that with probability 1-$\epsilon$, the action is chosen according to $\pi_{\rm BC}$, and with probability $\epsilon$, the action is chosen uniformly at random within the action range.

For every $\epsilon$, we generate $K$ policy rollouts and thus obtain $K \times d$ ranked demonstrations, where each trajectory is ranked based on the noise level that generated it, with trajectories considered of equal preference if generated from the same noise level. 
Thus, by generating rollouts from $\pi_{\rm BC}(\cdot | \epsilon)$ with varying levels of noise, we can obtain an arbitrarily large number of ranked demonstrations: 
\begin{equation}
D_{\rm ranked} = \{ \tau_i \prec \tau_j :  \tau_i\sim \pi_{\rm BC}(\cdot | \epsilon_i), \tau_j\sim \pi_{\rm BC}(\cdot | \epsilon_j), \epsilon_i > \epsilon_j\}.
\end{equation}
Given these ranked demonstrations, we then use T-REX to learn a reward function $\hat{R}$ from which we can optimize a policy $\hat{\pi}$ using any reinforcement learning algorithm (see Appendix~\ref{app:drex_implementation} for details).

\section{Experimental Results}

\subsection{Automatically generating rankings via noise}
To test whether injecting noise can create high-quality, automatic rankings, we used simulated suboptimal demonstrations from a partially trained reinforcement learning agent. To do so, we used the Proximal Policy Optimization (PPO) \cite{schulman2017proximal} implementation from OpenAI Baselines \cite{baselines} to partially train a policy on the ground-truth reward function. We then ran behavioral cloning on these demonstrations and plotted the degradation in policy performance for increasing values of $\epsilon$.  

We evaluated noise degradation on the Hopper and Half-Cheetah domains in MuJoCo and on the seven Atari games listed in Table~\ref{tab:D-REX}. 
To perform behavioral cloning, we used one suboptimal demonstration trajectory of length 1,000 for the MuJoCo tasks and 10 suboptimal demonstrations for the Atari games. We then varied $\epsilon$ and generated rollouts for different noise levels. We plotted the average return along with one standard deviation error bars in Figure~\ref{fig:atari_noise_degradation} (see Appendix~\ref{ref:noise_degredation} for details). We found that behavioral cloning with small noise tends to have performance similar to that of the average performance of the 
demonstrator. As noise is added, the performance degrades until it reaches the level of a uniformly random policy ($\epsilon = 1$). These plots validate our assumption that, in expectation, adding increasing amounts of noise will cause near-monotonic performance degradation. 

\begin{figure}[t]
    \centering
    \subfigure[Hopper]{
        \includegraphics[width=.24\linewidth]{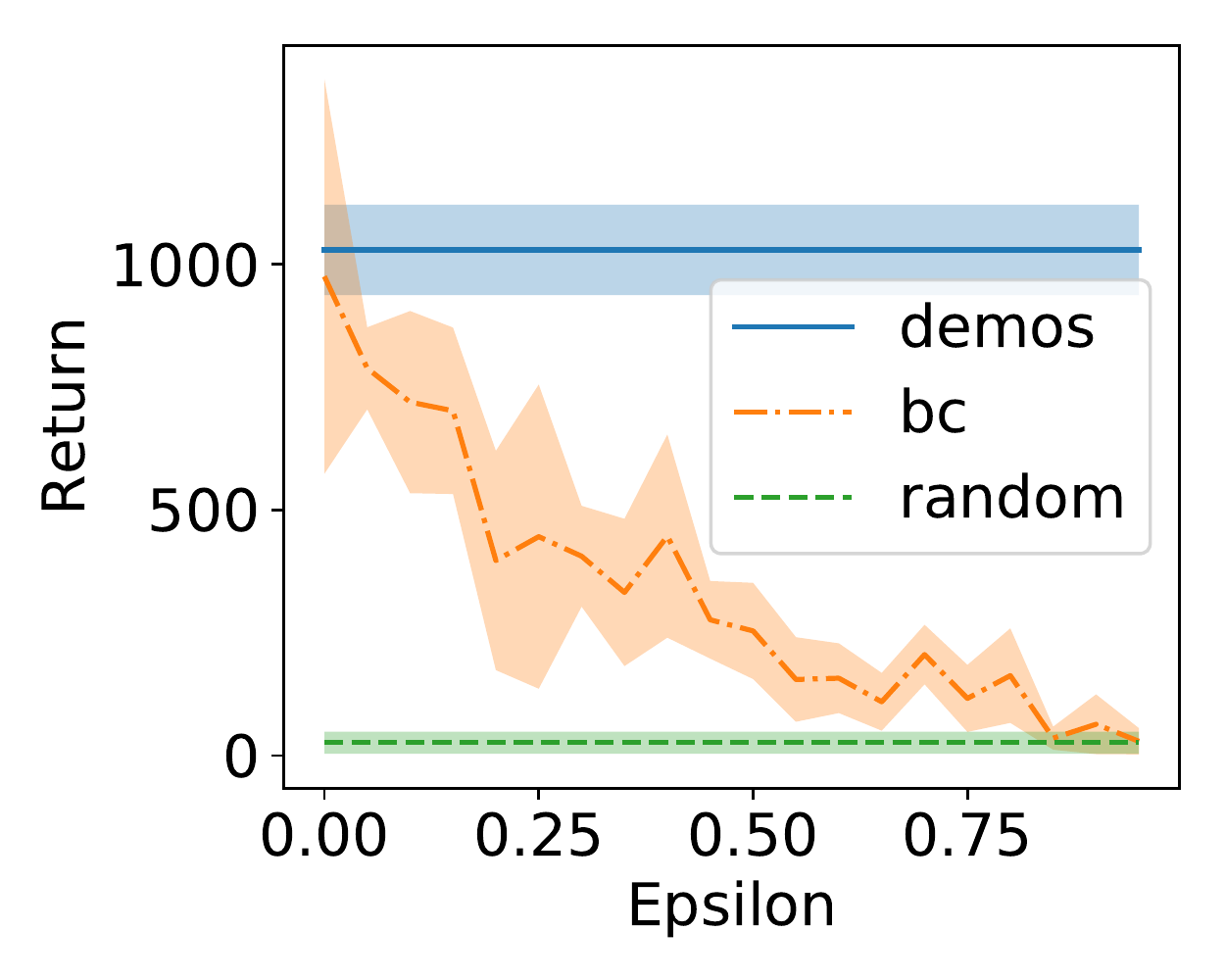}
    }
    \hspace{-0.35cm}
    \subfigure[Half-Cheetah]{
        \includegraphics[width=.24\linewidth]{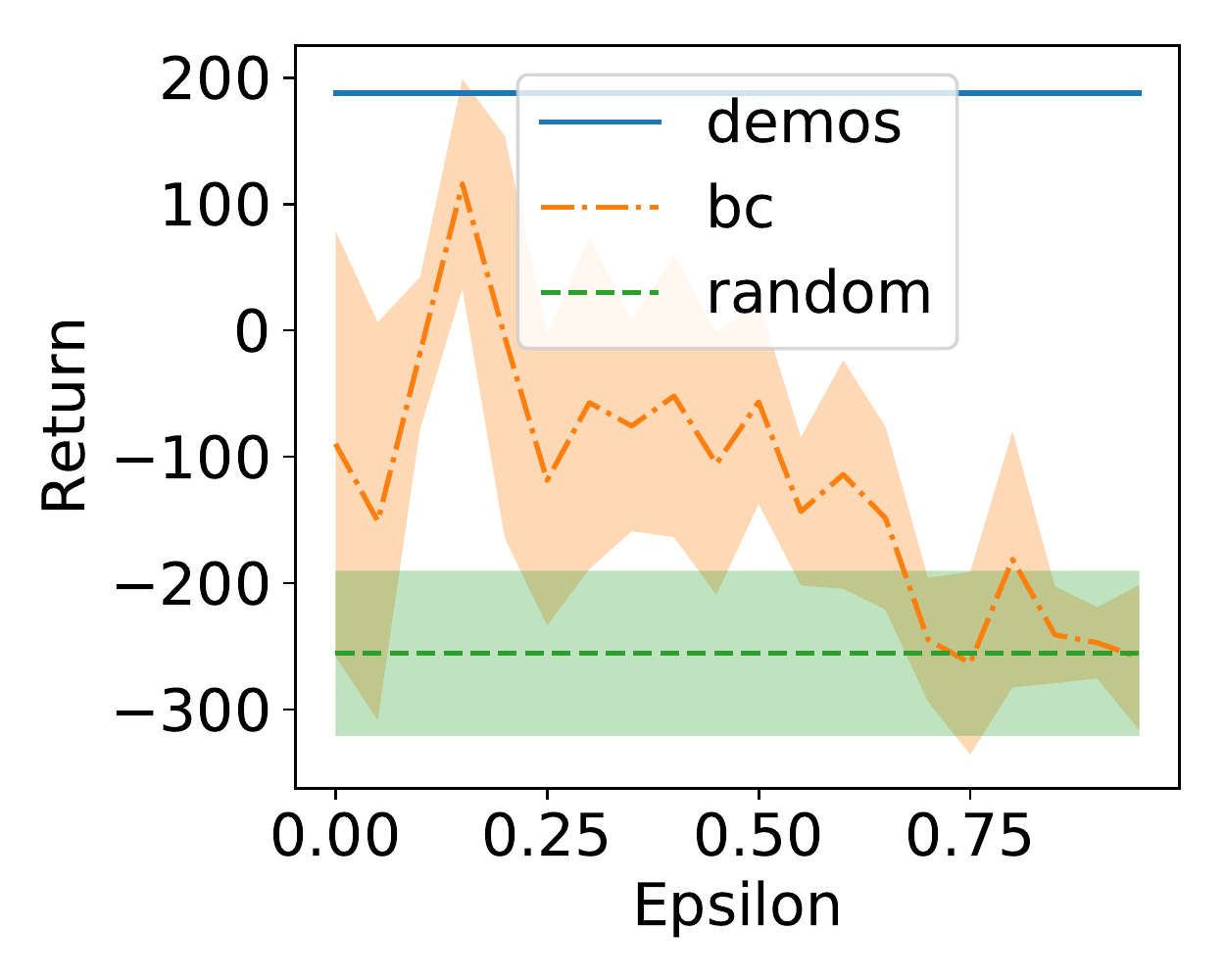}
    }
    \hspace{-0.35cm}
     \subfigure[Beam Rider]{
        \includegraphics[width=.24\linewidth]{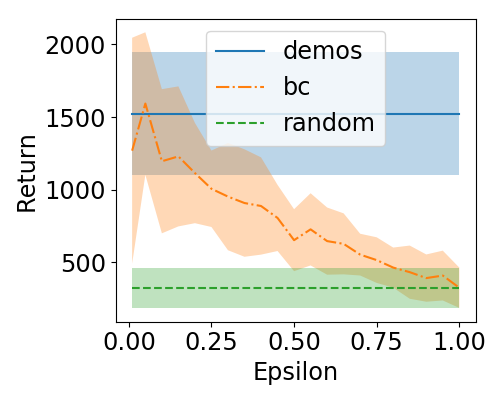}
    }
    \hspace{-0.35cm}
    \subfigure[Seaquest]{
        \includegraphics[width=.24\linewidth]{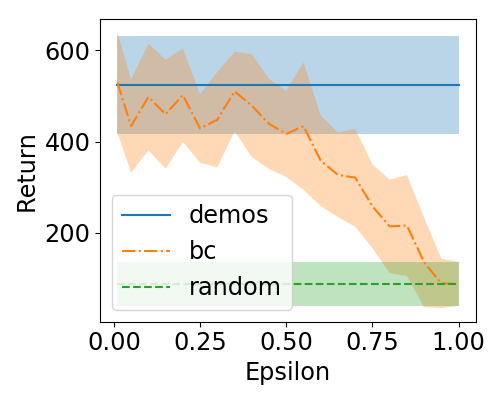}
    }
    \caption{Examples of the degradation in performance of an imitation policy learned via behavioral cloning as more noise is injected into the policy. Behavioral cloning is done on a 1,000-length trajectory (MuJoCo tasks) or 10 demonstrations (Atari games). Plots show mean and standard deviations over 5 rollouts (MuJoCo tasks) or 20 rollouts (Atari games).}
    \label{fig:atari_noise_degradation}
\end{figure}

\subsection{Reward extrapolation}
We next tested whether D-REX allows for accurate reward extrapolation. 
We used noise injection, as described in the previous section, to generate 100 synthetically-ranked demonstrations. For MuJoCo, we used the noise schedule consisting of 20 different noise levels, evenly spaced over the interval $[0,1)$ and generated $K=5$ rollouts per noise level. For Atari, we used the noise schedule $\mathcal{E}=(1.0, 0.75, 0.5, 0.25, 0.02)$ with $K=20$ rollouts per noise level. By automatically generating ranked demonstrations, D-REX is able to leverage a small number of unranked demonstrations to generate a large dataset of ranked demonstrations for reward function approximation. We used the T-REX algorithm \cite{browngoo2019trex} to learn a reward function from these synthetically ranked demonstrations.

To investigate how well D-REX learns the true reward function, we evaluated the learned reward function  $\hat{R}_\theta$ on the original demonstrations and the synthetic demonstrations obtained via noise injection. We then compared the ground-truth returns with the predicted returns under $\hat{R}_\theta$. 
We also tested reward extrapolation on a held-out set of trajectories obtained from PPO policies that were trained longer on the ground-truth reward than the policy used to generate the demonstrations for D-REX. These additional trajectories allow us to measure how well the learned reward function can extrapolate beyond the performance of the original demonstrations. We scale all predicted returns to be in the same range as the ground-truth returns. 
The results for four of the tasks are shown in Figure~\ref{fig:atari_reward_extrapolation_partial}. The remaining plots are included in Appendix~\ref{app:extrapolate_heat}. The plots show relatively strong correlation between ground truth returns and predicted returns across most tasks, despite having no a priori access to information about true returns, nor rankings.
We also generated reward sensitivity heat maps \cite{greydanus2018visualizing} for the learned reward functions. These visualizations provide evidence that D-REX learns semantically meaningful features that are highly correlated with the ground truth reward. For example, on Seaquest, the reward function learns a shaped reward that gives a large penalty for an imminent collision with an enemy  (Appendix \ref{app:extrapolate_heat}).

\begin{figure}
    \centering
    \subfigure[Hopper]{
        \includegraphics[width=.24\linewidth]{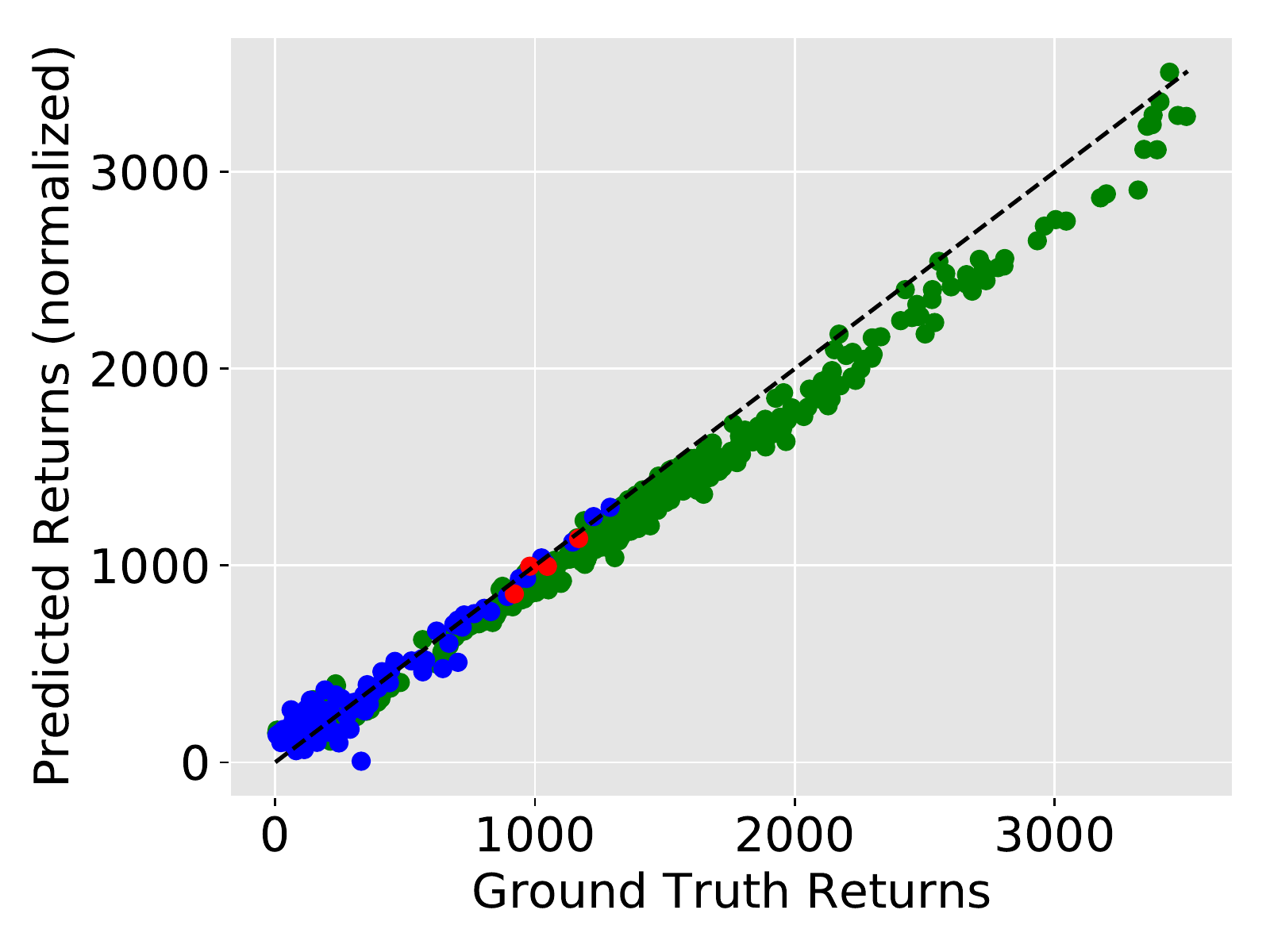}
        
    }
    \hspace{-0.3cm}
     \subfigure[Half-Cheetah]{
        \includegraphics[width=.24\linewidth]{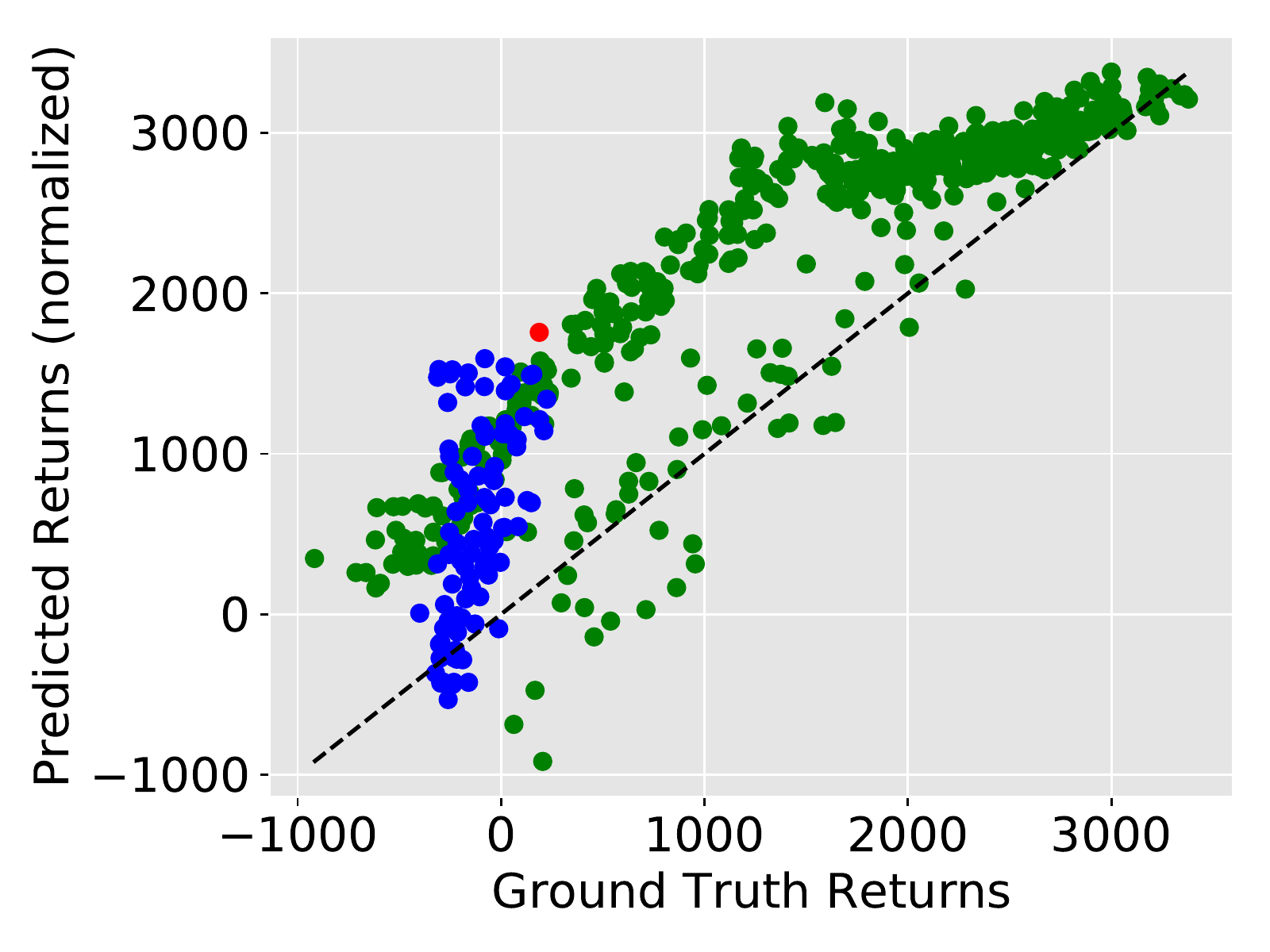}
        
    }
    \hspace{-0.3cm}
    \subfigure[Beam Rider]{
        \includegraphics[width=.24\linewidth]{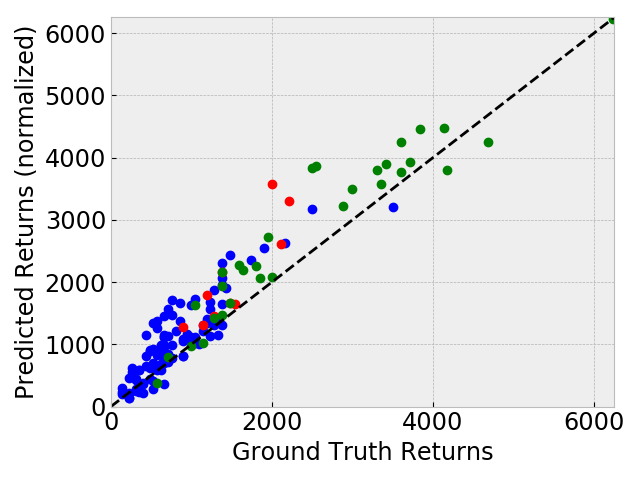}
        
    }
    \hspace{-0.3cm}
    \subfigure[Seaquest]{
        \includegraphics[width=.24\linewidth]{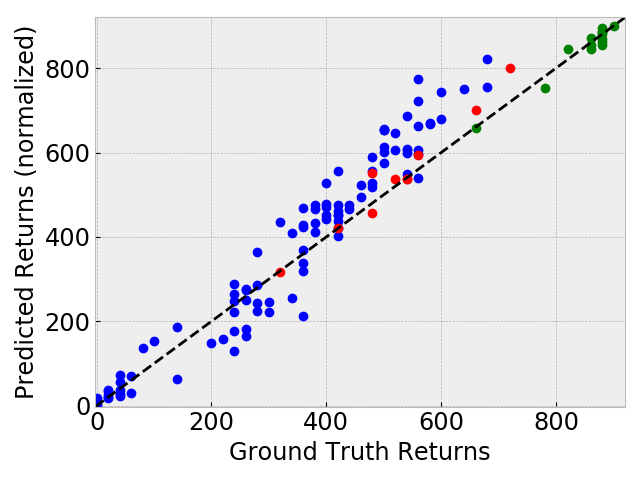}
    }
    \caption{Extrapolation plots for a selection of MuJoCo and Atari tasks (see the appendix for more plots). Blue dots represent synthetic demonstrations generated via behavioral cloning with different amounts of noise injection. Red dots represent actual demonstrations, and green dots represent additional trajectories not seen during training. We compare ground truth returns over demonstrations to the predicted returns from D-REX (normalized to be in the same range as the ground truth returns).}
    \label{fig:atari_reward_extrapolation_partial}
\end{figure}

\subsection{Extrapolating beyond the demonstrator's performance}

Lastly, we tested whether the reward functions learned using D-REX can be used in conjunction with deep reinforcement learning to achieve better-than-demonstrator performance. We ran PPO on the learned reward function $\hat{R}_\theta$ for 1 million timesteps (MuJoCo tasks) and 50 million frames (Atari games). We ran three replicates of PPO with different seeds and report the best performance on the ground-truth reward function, averaged over 20 trajectory rollouts. Table~\ref{tab:D-REX} compares the performance of the demonstrator with the performance of D-REX, behavioral cloning (BC), and Generative Adversarial Imitation Learning (GAIL) \cite{ho2016generative}, a state-of-the-art imitation learning algorithm. 

The results in Table~\ref{tab:D-REX} demonstrate that policies optimized using D-REX outperform the best demonstration in all tasks except for Pong. Furthermore, D-REX is also able to outperform BC and GAIL across all tasks except for Hopper and Pong. 
On the simulated MuJoCo robotics tasks, D-REX results in a 77\% (Hopper) and 418\% (HalfCheetah) performance increase when compared with the best demonstration.
On Q*Bert, D-REX exploits a known loophole in the game which allows nearly infinite points. Excluding Q*Bert, D-REX results in an average performance increase of 39\% across the Atari tasks, when compared with the best demonstration. 
To test the robustness of the policy learned via D-REX, we also considered the worst-case performance, something that is important for safe inverse reinforcement learning \cite{brown2018efficient,brown2018risk,majumdar2017risk}. As shown in Appendix~\ref{app:risk}, D-REX achieves better worst-case performance than either the demonstrator or standard imitation learning algorithms. 
To ensure that D-REX is learning more than a simple bonus for staying alive, we also compared D-REX with a PPO agent trained with a +1 reward for every timestep. Our results in Appendix~\ref{app:livelong} demonstrate that D-REX is superior to a simple +1 reward across all games, except for Pong. 

\begin{table*}
\caption{Comparison of the performance of D-REX with behavioral cloning (BC), GAIL \cite{ho2016generative}, and the demonstrator's performance. Results are the best average ground-truth returns over 3 random seeds with 20 trials per seed. Bold denotes performance that is better than the best demonstration.}
\label{tab:D-REX}
\vskip 0.15in
\begin{center}
\begin{small}
\begin{tabular}{ccccccccc}
\toprule
 & \multicolumn{2}{c}{Demonstrations} & \multicolumn{2}{c}{D-REX} &\multicolumn{2}{c}{BC} & \multicolumn{2}{c}{GAIL}\\
  \midrule
Task &  Avg. & Best & Average & Stdev. & Avg. & Stdev. & Avg. & Stdev.\\ 
\midrule 
Hopper & 1029.1 & 1167.9 & \textbf{2072.0} & (1574.2) & 943.8 & (208.4) & \textbf{2700.2} & (692.3)\\ 
HalfCheetah & 187.7 & 187.7 & \textbf{972.9} &(96.1) & -115.9 & (179.8) & 85.2 & (86.0)\\ 
\midrule 
Beam Rider & 1,524.0 & 2,216.0 & \textbf{7,220.0} &	(2221.9) & 1,268.6 & (776.6) & 1778.0	& (787.1)\\ 
Breakout & 34.5 & 59.0 & \textbf{94.7} & (16.5) & 29.75 & (10.1) & 0.3	 & (0.4)
 \\ 
Enduro & 85.5 & 134.0 & \textbf{247.9} &(88.4) & 83.4 & (27.0) & 62.4 & 	(24.0)
\\ 
Pong & 3.7 & \textbf{14.0} & -9.5 &	(9.8) & 8.6 & (9.5) & -3.4 &	(3.8)
 \\ 
Q*bert & 770.0 & 850.0 & \textbf{22543.8}	& (7434.1) & 1,013.75 & (721.1) & 737.5	& (311.4)
\\ 
Seaquest & 524.0 & 720.0 & \textbf{801.0} & (4.4) & 530.0 & (109.8) & 554.0	& (108.8)
 \\ 
Space Invaders & 538.5 & 930.0 & \textbf{1,122.5} &	(501.2) & 426.5 & (187.1) & 364.8 &	(139.7)
\\ 
\bottomrule
\end{tabular}
\end{small}
\end{center}
\end{table*}

\section{Conclusion}
Imitation learning approaches are typically unable to outperform the demonstrator. This is because most approaches either directly mimic the demonstrator or find a reward function that makes the demonstrator appear near optimal. While algorithms that can exceed the performance of a demonstrator exist, they either rely on a significant number of active queries from a human \cite{christiano2017deep,sadigh2017active,palan2019learning}, a hand-crafted reward function \cite{pastor2011skill,hester2018deep}, or pre-ranked demonstrations \cite{browngoo2019trex}. Furthermore, prior research has lacked theory about when better-than-demonstrator performance is possible. We first addressed this lack of theory by presenting a sufficient condition for extrapolating beyond the performance of a demonstrator. We also provided theoretical results demonstrating how preferences and rankings allow for better reward function learning by reducing the learner's uncertainty over the true reward function.

We next focused on making reward learning from rankings more applicable to a wider variety of imitation learning tasks where only unlabeled demonstrations are available. We presented a novel imitation learning algorithm, Disturbance-based Reward Extrapolation (D-REX) that automatically generates ranked demonstrations via noise injection and uses these demonstrations to seek to extrapolate beyond the performance of a suboptimal demonstrator. We empirically evaluated D-REX on a set of simulated robot locomotion and Atari tasks and found that D-REX outperforms state-of-the-art imitation learning techniques and also outperforms the best demonstration in 8 out of 9 tasks. These results  provide the first evidence that better-than-demonstrator imitation learning is possible without requiring extra information such as rewards, active supervision, or preference labels. Our results open the door to the application of a variety of ranking and preference-based learning techniques \cite{cao2007learning,chen2009ranking} to standard imitation learning domains where only unlabeled demonstrations are available.



\acknowledgments{This  work  has  taken  place  in  the  Personal  Autonomous Robotics Lab (PeARL) at The University of Texas at Austin. PeARL  research  is  supported  in  part  by  the  NSF  (IIS-1724157, IIS-1638107, IIS-1617639, IIS-1749204) and ONR(N00014-18-2243).
}


{
\small
\bibliography{roar_refs}  
}

\appendix

\section{Theory and Proofs}
We consider the case where the reward function of the demonstrator is approximated by a linear combination of features $R(s)  = w^T \phi(s)$. Note that these can be arbitrarily complex features, such as the activations of the penultimate layer of a deep neural network. The expected return of a policy when evaluated on $R(s)$ is given by 
\begin{equation}
J(\pi|R) = \mathbb{E}_\pi\bigg[\sum_{t=0}^\infty \gamma^t R(s_t)\bigg] = w^T\mathbb{E}_\pi\bigg[\sum_{t=0}^\infty \gamma^t \phi(s_t)\bigg] = w^T \Phi_\pi,
\end{equation}
where $\Phi_\pi$ are the expected discounted feature counts that result from following the policy $\pi$.

\setcounter{theorem}{0}
\setcounter{proposition}{0}
\setcounter{corollary}{0}
\begin{theorem}
If the estimated reward function is $\hat{R}(s) = w^T \phi(s)$ and the true reward function is $R^*(s) = \hat{R}(s) + \epsilon(s)$ for some error function $\epsilon : \mathcal{S} \rightarrow \mathbb{R}$ and $\|w\|_1 \leq 1$, then extrapolation beyond the demonstrator, i.e., $J(\hat{\pi}|R^*) > J(\mathcal{D}| R^*)$, is guaranteed if: 
\begin{equation}
J(\pi^*_{R^*}|R^*) - J(\mathcal{D}|R^*) > \epsilon_\Phi + \frac{2 \|\epsilon\|_\infty}{1 - \gamma}
\end{equation}
where $\pi^*_{R^*}$ is the optimal policy under $R^*$, $\epsilon_\Phi = \|\Phi_{\pi^*} - \Phi_{\hat{\pi}}\|_\infty$ and $\|\epsilon\|_\infty=\sup\left\{\,\left|\epsilon(s)\right|:s\in \mathcal{S}\,\right\}$.
\end{theorem}
\begin{proof}
In order for extrapolation to be possible, the demonstrator must perform worse than $\hat{\pi}$, the policy learned via IRL, when evaluated under the true reward function. We define $\delta = J(\pi^*_{R^*} | R^*) - J(\mathcal{D}|R^*)$ as the optimality gap between the demonstrator and the optimal policy under the true reward function. We want to ensure that $J(\pi^* | R^*) - J(\hat{\pi}|R^*) < \delta$. We have 
\begin{eqnarray}
  J(\pi^*_{R^*} | R^*) - J(\hat{\pi}|R^*) &=& \bigg| \mathbb{E}_{\pi^*} \big[\sum_{t=0}^\infty \gamma^t R^*(s_t) \big] - \mathbb{E}_{\hat{\pi}} \big[\sum_{t=0}^\infty \gamma^t R^*(s_t) \big] \bigg| \\
 &=& \bigg| \mathbb{E}_{\pi^*} \big[\sum_{t=0}^\infty \gamma^t (w^T \phi(s_t) + \epsilon(s_t)) \big] - \mathbb{E}_{\hat{\pi}} \big[\sum_{t=0}^\infty \gamma^t (w^T \phi(s_t) + \epsilon(s_t)) \big] \bigg| \\
 &=& \bigg| w^T\mathbb{E}_{\pi^*} \big[\sum_{t=0}^\infty \gamma^t \phi(s_t)\big] + \mathbb{E}_{\pi^*} \big[ \sum_{t=0}^\infty \gamma^t \epsilon(s_t) \big] - w^T\mathbb{E}_{\hat{\pi}} \big[\sum_{t=0}^\infty \gamma^t \phi(s_t) \big] - \mathbb{E}_{\hat{\pi}} \big[\sum_{t=0}^\infty \gamma^t \epsilon(s_t) \big] \bigg| \nonumber\\ \\
  &=& \bigg| w^T \Phi_{\pi^*} + \mathbb{E}_{\pi^*} \big[\sum_{t=0}^\infty \gamma^t \epsilon(s_t) \big] - w^T \Phi_{\hat{\pi}} - \mathbb{E}_{\hat{\pi}} \big[\sum_{t=0}^\infty  \gamma^t \epsilon(s_t) \big] \bigg| \\
  &=& \bigg| w^T (\Phi_{\pi^*} - \Phi_{\hat{\pi}}) + \mathbb{E}_{\pi^*} \big[\sum_{t=0}^\infty \gamma^t \epsilon(s_t) \big]  - \mathbb{E}_{\hat{\pi}} \big[\sum_{t=0}^\infty \gamma^t \epsilon(s_t) \big] \bigg| \\
  &\leq& \bigg| w^T (\Phi_{\pi^*} - \Phi_{\hat{\pi}}) + \big[\sum_{t=0}^\infty \gamma^t \sup_{s \in \mathcal{S}} \epsilon(s) \big]  -  \big[\sum_{t=0}^\infty \gamma^t \inf_{s\in \mathcal{S}}  \epsilon(s) \big] \bigg| \\
  &=& \bigg| w^T (\Phi_{\pi^*} - \Phi_{\hat{\pi}}) + \big(\sup_{s \in \mathcal{S}} \epsilon(s) -  \inf_{s\in \mathcal{S}} \epsilon(s)\big) \sum_{t=0}^\infty \gamma^t  \bigg| \\
    &\leq& \bigg| w^T (\Phi_{\pi^*} - \Phi_{\hat{\pi}}) + \frac{2 \|\epsilon \|_\infty}{1 - \gamma} \bigg| \\
    &\leq& | w^T (\Phi_{\pi^*} - \Phi_{\hat{\pi}})| + \bigg|\frac{2 \|\epsilon \|_\infty}{1 - \gamma} \bigg| \\
&\leq&  \|w\|_1 \|\Phi_{\pi^*} - \Phi_{\hat{\pi}}\|_\infty + \frac{2 \|\epsilon \|_\infty}{1 - \gamma} \label{line:holders}\\
&\leq&  \epsilon_\Phi + \frac{2 \|\epsilon \|_\infty}{1 - \gamma} 
\end{eqnarray}
where $\Phi_\pi = \mathbb{E}_\pi [\sum_{t=0}^\infty \gamma^t \phi(s_t)]$ and line (\ref{line:holders}) results from H{\"o}lder's inequality.
Thus, as long as $\delta > \epsilon_\Phi - \frac{2 \|\epsilon\|_\infty}{1 - \gamma}$, then $J(\pi^*|R^*) - J(\hat{\pi}|R^*) < J(\pi^*|R^*) - J(\mathcal{D}|R^*)$ and thus, $J(\hat{\pi}|R^*) > J(\mathcal{D}|R^*)$.
\end{proof}

Theorem 1 makes the assumption that the learner and demonstrator operate in the same state-space. However, because Theorem 1 only involve differences in expected features and reward errors over state visitations, the transition dynamics or action spaces are not required to be the same between the demonstrator and the learner. While our experiments use demonstrator actions to perform behavioral cloning, future work could use an imitation learning from observation method \cite{torabi2018behavioral} to learn an initial cloned policy without requiring demonstrator actions. This would make D-REX robust to action differences between the demonstrator and learner and allow the learner to potentially learn a more efficient policy if its action space and transition dynamics are better suited to the task.

\subsection{Extrapolation via ranked demonstrations}
The previous results demonstrate that in order to extrapolate beyond a demonstrator, it is sufficient to have small reward approximation error. However, the following proposition, adapted with permission from \cite{pabloranked}, demonstrates that the reward functions inferred by an IRL or apprenticeship learning algorithm may be quite superficial and may not accurately represent some dimensions of the true reward function.

\begin{proposition}\label{prop:mdp_example}
There exist MDPs with true reward function $R^*$, expert policy $\pi_E$, approximate reward function $\hat{R}$, and non-expert policies $\pi_1$ and $\pi_2$, such that 
\begin{eqnarray}
&&\pi_E = \arg \max_{\pi \in \Pi} J(\pi |R^*) \;\text{  and  }\; J(\pi_1 | R^*) \ll J(\pi_2 | R^*) \label{eqn:mdpexample1_apx} \\
&&\pi_E = \arg \max_{\pi \in \Pi} J(\pi | \hat{R}) \;\text{  and  }\; J(\pi_1 | \hat{R}) = J(\pi_2 | \hat{R}) \label{eqn:mdpexample2_apx}. 
\end{eqnarray}
However, enforcing a preference ranking over trajectories, $\tau^* \succ \tau_2 \succ \tau_1$, where $\tau^* \sim \pi^*$, $\tau_2 \sim \pi_2$, and $\tau_1 \sim \pi_1$, results in a learned reward function $\hat{R}$, such that 
\begin{equation}
\pi_E = \arg \max_{\pi \in \Pi} J(\pi |\hat{R}) \;\text{  and  }\; J(\pi_1 | \hat{R}) <  J(\pi_2 | \hat{R}) \label{eqn:mdpexample3_apx}.
\end{equation}

\end{proposition}
\begin{proof}
Consider the MDP shown below. 
\begin{figure}[h]
\centering
\includegraphics[scale=0.25]{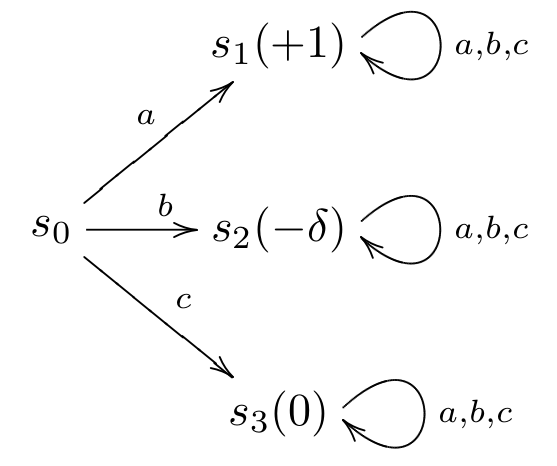}
\label{fig:mdp_proof_by_picture}
\end{figure}
There are three actions $a$, $b$, $c$, with deterministic transitions. Each transition is labeled by the action name. The true reward received upon entering a state is indicated in parenthesis, and $\delta \gg 0$ is some arbitrary constant. Clearly, $\pi_E(s_0) = a$. Setting $\hat{R}(s_1) = 1$, $\hat{R}(s_2) = \hat{R}(s_3) = 0$, $\pi_1(s_0) = b$, and $\pi_2(s_0) = c$ provides the existence proof for Equations (\ref{eqn:mdpexample1_apx}) and (\ref{eqn:mdpexample2_apx}).

Enforcing the preference constraints $\tau^* \succ \tau_1 \succ \tau_2$ for $\tau^* = (s_0, a, s_1)$, $\tau_2 = (s_0,b,s_2)$, $\tau_1 = (s_0,c,s_3)$, results in a learned reward function $\hat{R}$ such that  $J(\tau^*|\hat{R}) > J(\pi_2 | \hat{R}) \succ J(\pi_1|\hat{R})$ which finishes the proof.
\end{proof}

Proposition~\ref{prop:mdp_example} gives a simple example of when learning from an expert demonstration reveals little about the underlying reward structure of the MDP. While it is true that solving the MDP in the above example only requires knowing that state $s_1$ is preferable to all other states, this will likely lead to an agent assuming that both $s_2$ and $s_3$ are equally undesirable. 

This may be problematic for several reasons. The first problem is that learning a reward function from demonstrations is typically used as a way to generalize to new situations---if there is a change in the initial state or transition dynamics, an agent can still determine what actions it should take by transferring the learned reward function. However, if the learned reward function is drastically different than the true reward, this can lead to poor generalization, as would be the case if the dynamics in the above problem change and action $a$ now causes a transition to state $s_2$. Another problem is that most learning from demonstration applications focus on providing \textit{non-experts} the ability to program by example. Thus, the standard IRL approach of finding a reward function that maximizes the likelihood of the demonstrations may lead to reward functions that overfit to demonstrations and may be oblivious to important differences in rewards.

A natural way to alleviate these problems is via ranked demonstrations.
Consider the problem of learning from a sequence of $m$ demonstrated trajectories $\tau_1, \ldots, \tau_m$, ranked according to preference such that $\tau_1 \prec \tau_2, \ldots \prec \tau_m$. Using a set of strictly ranked demonstrations avoids the degenerate all-zero reward. Furthermore, ranked demonstrations provide explicit information about both what to do as well as what not to do in an environment.

\subsection{Ranking Theory} \label{app:fullranking}
IRL is an ill-posed problem due to reward ambiguity---given any policy, there are an infinite number of reward functions that make this policy optimal \cite{ng1999policy}. However, it is possible to analyze the amount of reward ambiguity. In particular, if the reward is represented by  a linear combination of weights, then the feasible region of all reward functions that make a policy optimal can be defined as an intersection of half-planes \cite{brown2019machine}: 
\begin{equation}
H_\pi = \bigcap_{\pi' \in \Pi} w^T (\Phi_\pi - \Phi_{\pi'}) \geq 0, 
\end{equation}

We define the \textit{reward ambiguity}, $G(H_{\pi})$, as the volume of this intersection of half-planes: 
\begin{equation}
G(H_{\pi}) = \text{Volume}(H_\pi),
\end{equation}
where we assume without loss of generality that $\|w\| \leq 1$, to ensure this volume is bounded.

We now prove that a total ranking over policies results in no more reward ambiguity than simply using the optimal policy.
\begin{proposition}
Given a policy class $\Pi$, an optimal policy $\pi^* \in \Pi$ and a total ranking over $\Pi$, and a reward function $R(s) = w^T \phi(s)$,  the reward ambiguity resulting from $\pi^*$ is greater than or equal to the reward ambiguity of using a total ranking, i.e., $G(H_\pi^*) \geq G(H_{\rm ranked})$. 
\end{proposition}
\begin{proof}
Consider policies $\pi_1$ and $\pi_2$ where $J_{R^*}(\pi_1) \geq J_{R^*}(\pi_2)$. We can write this return inequality in terms of half-spaces as follows:
\begin{eqnarray}
&& J_{R^*}(\pi_1) \geq J_{R^*}(\pi_2) \\
&\iff& w^T \Phi_{\pi_1} \geq w^T \Phi_{\pi_2}\\
&\iff& w^T (\Phi_{\pi_1} - \Phi_{\pi_2}) \geq 0
\end{eqnarray}
defining a half-space over weight vectors.

Consider the optimal policy $\pi^*_{r^*}$. This policy induces a set of half-space constraints over all other possible policies $\pi \in \Pi$. Thus we have the following half-space constraints:
\begin{eqnarray}
H_{\pi^*} = \bigcap_{\pi \in \Pi} w^T (\Phi_{\pi^*} - \Phi_{\pi}) \geq 0
\end{eqnarray}

However, if we have a total ordering over $\Pi$, then we have the following intersection of half-spaces
\begin{eqnarray}
H_{\rm ranked}  &=& \bigcap_{\pi_i \succsim \pi_j \in \Pi } w^T (\Phi_{\pi_i} - \Phi_{\pi_j}) \geq 0  \\
&=& H_{\pi^*} \cap \bigg(\bigcap_{\substack{\pi_i \succsim \pi_j \in \Pi \\ \pi_i \neq \pi^*}} w^T (\Phi_{\pi_i} - \Phi_{\pi_j}) \geq 0 \bigg).
\end{eqnarray}

Thus, $H_{\rm ranked} \subseteq H_{\pi^*}$, and the volume of the set of feasible reward functions induced by the total ranking is therefore less than or equal to the volume of $H_{\pi^*}$, i.e., $G(H_{\rm ranked}) \leq G(H_{\pi^*})$.

\end{proof}

\subsection{Uncertainty Reduction for Random Halfspaces} \label{app:halfspace}
In this section we seek to bound how many reward hypotheses are removed after a certain number of ranked trajectories, $\tau_i \prec \tau_j$. Our main result that we will prove is that sampling random half-space constraints causes the reward ambiguity to decrease exponentially. 

We assume that the true reward function, $R^*(s) = {w^*}^T \phi(s)$, is a linear combination of features $w \in \mathbb{R}^n$. Let $\mathcal{H}$ denote the set of all reward hypotheses. We assume a linear combination of features $R(s) = w^T \phi(s)$, so $\mathcal{H} = \mathbb{R}^n$. We also make the common assumption that $\|w\|_1 \leq 1$ so that $\mathcal{H}$ is bounded \cite{abbeel2004apprenticeship,brown2018efficient}. 

Each pair of trajectories, $\tau_i \prec \tau_j$, that are ranked based on $R^*$, forms a half-space constraint over the true reward feature weights ${w^*}^T$:
\begin{eqnarray}
&\tau_i \prec \tau_j \\
\Rightarrow& J(\tau_i | R^*) < J(\tau_j | R^*)\\ \Rightarrow& {w^*}^T \Phi_{\tau_i} < {w^*}^T \Phi_{\tau_j} \\
\Rightarrow& {w^*}^T (\Phi_{\tau_j} - \Phi_{\tau_i}) > 0,
\end{eqnarray}
where $\Phi_\tau = \sum_{t =0}^T \gamma^t \phi(s_t)$ for $\tau = (s_0, a_0, s_1, a_1, \ldots, s_T, a_T)$.

We define $\Xi$ to be the set of all trajectories in the MDP. Let $\mathcal{X}$ be the set of half-space constraints that result from all ground-truth rankings over all possible trajectory pairs from $\Xi$, i.e., 
\begin{equation}
\mathcal{X} = \big\{ (\Phi_{\tau_j} - \Phi_{\tau_i}) : J(\tau_j|R^*) > J(\tau_j | R^*), \tau_i, \tau_j \in \Xi \big\}.
\end{equation}

We define the space of all reward feature weight vectors that are consistent with $\mathcal{X}$ as $\mathcal{R}$, where 
\begin{equation}
\mathcal{R} = \bigg\{ w : w \in \bigcap_{x \in \mathcal{X}} w^T x > 0 \bigg\},
\end{equation}
where $x$ is one of the half-space normal vectors from the set $\mathcal{X}$.
We also define $\mathcal{J} = \mathcal{H} \setminus \mathcal{R}$, i.e., the space that gets cut by all the half-spaces defined by the normal vectors in $\mathcal{X}$. Our goal is to find a lower bound, in expectation, on how much volume is cut from $\mathcal{J}$ for $k$ half-space constraints. 

To simplify our proofs and notation, we assume there are only a finite number of reward weight hypotheses in $\mathcal{J}$. Denote $j$ as the $j$th reward hypothesis in $\mathcal{J}$ and let $J = |\mathcal{J}|$ be the total number of these hypotheses. For each halfspace $x \in \mathcal{X}$, define $Z_i(x)$ as the binary indicator of whether halfspace $x$ (assumed to be represented by its normal vector) eliminates the hypothesis $j \in \mathbb{R}^n$, i.e.
\begin{equation}
Z_j(x) = 
\begin{cases}
1 & \text{ if } j^T x \leq 0 \\
0 & \text{otherwise}.
\end{cases}
\end{equation}

We define 
\begin{equation}
\mathcal{T} = \{j^{(\sum_{x \in \mathcal{X}} Z_j(x))} :  j \in \mathcal{J} \}
\end{equation}
 to be the multiset containing all the hypotheses in $\mathcal{J}$ that are eliminated by the half-space constraints in $\mathcal{X}$. This is a multiset since we include repeats of $j$ that are eliminated by different half-spaces in $\mathcal{X}$. We use the notation $\{ a^{(y)} \}$ to denote the multiset with $y$ copies of $a$. We define $T = |\mathcal{T}|$ to be the sum of the total number of reward hypotheses that are eliminated by each half-space in $\mathcal{X}$. Similarly, we define $H = |\mathcal{H}|$ and $X = |\mathcal{X}|$ to be the number of hypothesis in $\mathcal{H}$ and the number of half-spaces in $\mathcal{X}$, respectively.

Our goal is to find a lower bound on the expected number of hypotheses that are removed given $k$ half-space constraints that result from $k$ pairwise trajectory preferences.

Note that we can define $T$ as 
\begin{equation}\label{eqn:Tdef}
T =  \sum_{j \in J} \sum_{x \in X} Z_j(x).
\end{equation}
We can derive the expected number of unique hypotheses in $\mathcal{J}$ that are eliminated by the first half-space as follows:
\begin{eqnarray}
\mathbb{E} [ \sum_{j \in \mathcal{J}} Z_j(x) | x \sim \mathcal{U}(\mathcal{X})] &=& \sum_{j \in \mathcal{J}} \mathbb{E}_{x \in \mathcal{X}} [Z_j(x)] \\
&=& \sum_{j \in \mathcal{J}} \sum_{x \in X} p(x) Z_j(x) \\
&=& \sum_{j \in \mathcal{J}} \sum_{x \in \mathcal{X}} \frac{1}{X} Z_j(x) \\
&=& \frac{1}{X} \sum_{j \in \mathcal{J}} \sum_{x \in \mathcal{X}} Z_j(x) \\
&=& \frac{T}{X}
\end{eqnarray}
where $\mathcal{U}(\mathcal{X})$ represents a uniform distribution over $\mathcal{X}$.
There are $T$ hypotheses left to eliminate and $X$ halfspaces left to choose from. Once we choose all of them we have eliminated all of $\mathcal{T}$, so, in expectation, each half-space elminates $T/X$ of the hypotheses in $\mathcal{T}$.
We can now consider $T$ to be the total number of hypotheses that remain to be eliminated after already enforcing some number of half-space constraints, and the same reasoning applies. We will take advantage of this later to form a recurrence relation over the elements left in $\mathcal{J}$, i.e. the reward function hypotheses that are false, but have not yet been eliminated by a pairwise trajectory ranking.

When selecting a half-space from $\mathcal{X}$ uniformly, we eliminate $T/X$ unique hypotheses in $\mathcal{J}$ in expectation. If we assume that each $j \in \mathcal{J}$ is distributed uniformly across the half-spaces, then on average, there are $T/J$ copies of a hypothesis $j$ in $\mathcal{T}$. This can be shown formally as follows: 
\begin{eqnarray} \label{eqn:uniformJi}
\mathbb{E}_{j \in \mathcal{J}} [ \sum_{x \in \mathcal{X}} Z_j(x)] 
&=& \sum_{j \in \mathcal{J}} p(j) \bigg[ \sum_{x \in \mathcal{X}}  Z_j(x) \bigg] \\
 &=& \frac{1}{J} \sum_{j \in \mathcal{J}}\sum_{x \in \mathcal{X}}  Z_j(x) \\ 
 &=& \frac{T}{J}.
\end{eqnarray}

Thus, selecting a half-space uniformly at random eliminates $T/X$ unique hypotheses from $\mathcal{J}$; however, there are $T/J$ copies of each hypothesis on average, so a random half-space eliminates $\frac{T^2}{JX}$ total hypotheses from $\mathcal{T}$ on average. 


\paragraph{Recurrence relation}
The above analysis results in the following system of recurrence relations where $J_k$ is the number of unique hypotheses that have not yet been eliminated after $k$ half-spaces constraints, $T_k$ is the sum of the total hypotheses (including repeats) that have not yet been eliminated the first $k$ half-spaces, and $X$ is the total number of half-space constraints.
\begin{eqnarray}
&&J_0 = J \\
&&T_0 = T \\
&&J_{k+1} = J_k - \frac{T_k}{X-k} \\
&&T_{k+1} = T_k - \frac{T_k^2}{J_k \cdot (X -k)}
\end{eqnarray}

To simplify the analysis we make the following assumptions.
First we assume that $\mathcal{R}$ is small compared to the full space of $\mathcal{H}$. Thus, $\mathcal{H} \approx \mathcal{J}$. Since each $x \in \mathcal{X}$ covers $H/2$ hypotheses, then each $x \in \mathcal{X}$ covers approximately $J/2$ hypotheses. Summing over all halfspaces, we have that $T_0  \approx \sum_{x \in \mathcal{X}} J/ 2 = XJ / 2$. Furthermore, if $X$, the number of possible half-spaces, is large, then $X - k \approx X$.

Given these assumptions we now have the following simplified system of recurrence relations:
\begin{eqnarray}
&&J_0 = J \\
&&T_0 = \frac{J_0 X}{2}\\
&&J_{k+1} = J_k - \frac{T_k}{X}\\
&&T_{k+1} = T_k - \frac{T_k^2}{J_k X}
\end{eqnarray}
where $T_k$ and $J_k$ are the number of elements in $\mathcal{T}$ and $\mathcal{J}$ after $k$ half-spaces constraints have been applied.

The solution to these recurrences is:
\begin{eqnarray}
&&J_k = \frac{J_0}{2^k}, \\
&&T_k = \frac{J_0 X}{2^{k+1}}.
\end{eqnarray}
and can be verified by plugging the solution into the above equations and checking that it satisfies the base cases and recurrence relations.


We can now bound the number of half-space constraints (the number of ranked random trajectory pairs) that are needed to get $J = |\mathcal{J}|$ down to some level $\epsilon$. 

\begin{theorem}
To reduce the volume of $\mathcal{J}$ such that $J_k = \epsilon$, then it suffices to have $k$ random half-space constraints, where
\begin{equation}
k = \log_2 \frac{J_0}{\epsilon}
\end{equation} 
\end{theorem}
\begin{proof}
As shown above, we have that 
\begin{equation}
J_k = \frac{J_0}{2^k}.
\end{equation}
Solving for $k$ and substituting $\epsilon$ for $J_k$ we get
\begin{equation}
k = \log_2 \frac{J_0}{\epsilon}
\end{equation} 
\end{proof}

This means that $J_k$ decreases exponentially with $k$, meaning we only need to sample a logarithmic number of half-space constraints to remove a large portion of the possible reward hypotheses.

\begin{corollary}
To reduce $J$ by $x$\% it suffices to have 
\begin{equation}
k = \log_2 \left(\frac{1}{1-x/100}\right)
\end{equation}
random half-space constraints.
\end{corollary}
\begin{proof}
We want
\begin{equation}
\frac{x}{100} = \frac{J_0 - \frac{J_0}{2^k}}{J_0} = 1 - \frac{1}{2^k}
\end{equation}
Solving for $k$ results in 
\begin{equation}
k = \log_2 \left(\frac{1}{1-x/100}\right).
\end{equation}
\end{proof}


\section{Noise Degradation} \label{ref:noise_degredation}
The full set of noise degradation plots for all seven Atari games are shown in Figure~\ref{fig:full_atari_noise_degredation}. For MuJoCo experiments, we used 20 different noise levels evenly spaced over the interval $[0.0,1.0)$ and generated 5 trajectories for each level. For the Atari experiments, we used the noise schedule $\mathcal{E} = (0.01, 0.25, 0.5, 0.75, 1.0)$ and generated $K=20$ trajectories for each level.

For the MuJoCo tasks, we used the following epsilon greedy policy:
\begin{equation}\label{eqn:epsilon_greedy_mujoco}
    \pi_{\rm BC}(s_t | \epsilon) =  
    \begin{cases} 
      \pi_{\rm BC}(s_t),& \text{with probability } 1-\epsilon\\
      a_t \sim \mathcal{U}([-1,1]^n), & \text{with probability } \epsilon.
  \end{cases}
\end{equation}

For Atari, we used the following epsilon greedy policy:
\begin{equation}\label{eqn:epsilon_greedy_atari}
    \pi_{\rm BC}(a_t | s_t, \epsilon) =  
    \begin{cases} 
      1 - \epsilon + \frac{\epsilon}{|\mathcal{A}|},& \text{if } \pi_{\rm BC}(s_t)=a_t\\
      \frac{\epsilon}{|\mathcal{A}|}, & \text{otherwise}
  \end{cases}
\end{equation}
where $|\mathcal{A}|$ is the number of valid discrete actions in each environment.

\begin{figure}[h]
    \centering
    \subfigure[Beam Rider]{
        \includegraphics[width=.3\linewidth]{figs/beamrider_degredation_plot.png}
        
    }
     \subfigure[Breakout]{
        \includegraphics[width=.3\linewidth]{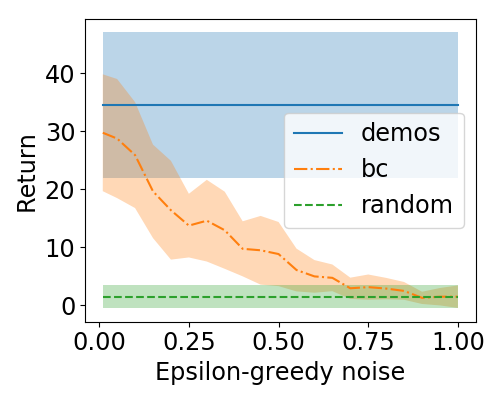}
        
    }
    \subfigure[Enduro]{
        \includegraphics[width=.3\linewidth]{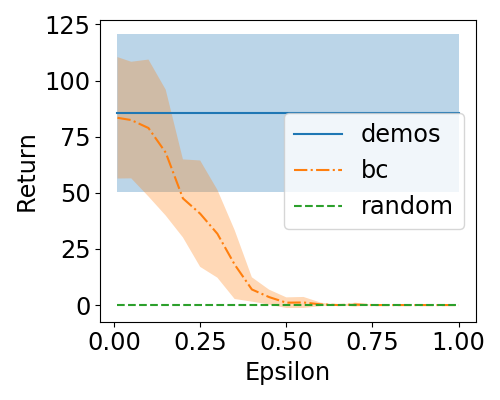}
        
    }
     \subfigure[Pong]{
        \includegraphics[width=.3\linewidth]{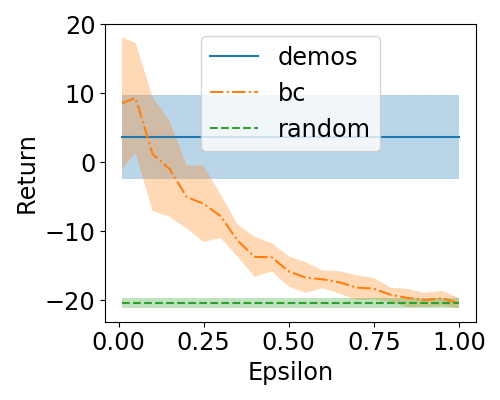}
        
    }
     \subfigure[Q*bert]{
        \includegraphics[width=.3\linewidth]{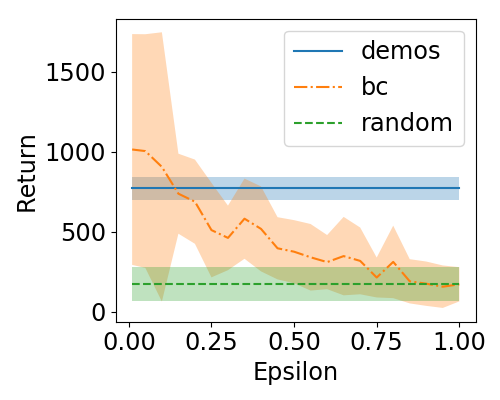}
        
    }
    \subfigure[Seaquest]{
        \includegraphics[width=.3\linewidth]{figs/seaquest_degredation_plot.png}
        
    }
    \subfigure[Space Invaders]{
        \includegraphics[width=.3\linewidth]{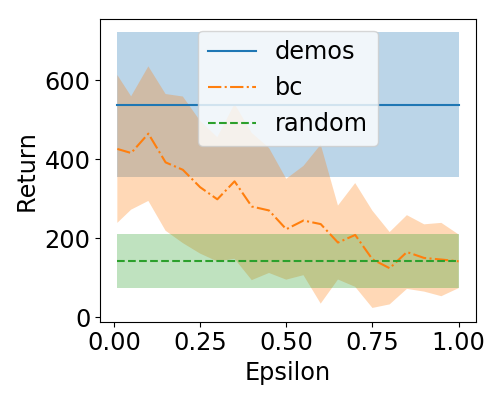}
        
    }
    \caption{The performance degradation of an imitation policy learned via behavioral cloning as the probability of taking a random action increases. Behavioral cloning is done on 10 demonstrations. Plots show mean and standard deviations over 20 rollouts per noise level.}
    \label{fig:full_atari_noise_degredation}
\end{figure}

\subsection{Noise Injection Theory} \label{app:degredation}
Let $\epsilon$ be the probability of taking a random action. We assume that $\mathcal{A}$ is finite, and assume a finite horizon MDP with a horizon of $T$. 

\paragraph{Optimal policy}
We start by analyzing what happens when we inject noise into an optimal policy, $\pi^*$. We will inject noise using the $\epsilon$-greedy policy defined in Equation (\ref{eqn:epsilon_greedy_atari}).
We define $p_{\epsilon}$ to be the probability of taking an suboptimal action. For simplicity, we assume that there is only one optimal action at each state and that the states visited are independent of each other. This simplifies the analysis, but ignores the problems of compounding error that occur in sequential decision making tasks \cite{ross2010efficient}. We will address the impact of compounding errors below. 

Given the above assumptions we have 
\begin{equation}
p_\epsilon = \frac{\epsilon(|A|-1)}{|A|},
\end{equation}
where $A$ is the set of actions in the MDP.

If we define $X$ to be the random variable representing the number of suboptimal actions in a trajectory of length $T$, then we can then analyze the number of suboptimal actions in the epsilon greedy policy $\pi^*(\cdot|\epsilon)$ using a binomial distribution with $T$ trials and probability $p = p_\epsilon$.

This gives us $\mathbb{E}[X] = T p_\epsilon $. If we assume that $|A|$ is large, then $p_\epsilon \approx \epsilon$ and we have $\mathbb{E}[X] = T \epsilon $. Thus, as $\epsilon$ goes from 0 to 1, the expected number of suboptimal actions in the policy interpolates between 0 and $T$. Furthermore, a standard Hoeffding bound shows that for large $T$, i.e., long trajectories, the empirical number of suboptimal actions will concentrate tightly around the mean $T\epsilon$. Thus, as $\epsilon$ goes from 0 to 1, with high probability the empirical number of suboptimal actions will interpolate between 0 and $T$.

\paragraph{Suboptimal cloned policy}
Suppose that instead of injecting noise into the optimal policy, we now inject noise into a policy, $\pi_{\rm BC}$, that is cloned from suboptimal demonstrations. Define $\beta$ to be the probability that $\pi_{\rm BC}$ takes the optimal action in a state. Note that here we also simplify the analysis by assuming that $\beta$ is a constant that does not depend on the current state. We define $p^{opt}_{\rm BC}$ to be the probability of taking a suboptimal action in state $s$ if $\pi_{\rm BC}$ would take the optimal action in that state. Similarly, we define $p^{sub}_{\rm BC}$ to be the probability of taking a suboptimal action in state $s$ if $\pi_{\rm BC}$ would take a suboptimal action in that state.

As before we define $p_\epsilon$ to be the probability that $\pi_{\rm BC} ( \cdot | \epsilon)$ takes a suboptimal action. Thus, we have that 
\begin{equation}
p_\epsilon = \beta p^{opt}_{\rm BC} + (1-\beta) p^{sub}_{\rm BC},
\end{equation}  
where 
\begin{equation}
p^{opt}_{\rm BC} = \frac{\epsilon(|A|-1)}{|A|}
\end{equation}
and 
\begin{equation}
p^{sub}_{\rm BC} = 1 - \epsilon + \frac{\epsilon(|A|-1)}{|A|}.
\end{equation}

Thus, we have 
\begin{eqnarray}
p_\epsilon &=& \beta \cdot \frac{\epsilon(|A|-1)}{|A|} + (1 - \beta) \cdot (1 - \epsilon + \frac{\epsilon(|A|-1)}{|A|})\\
&=& (1-\beta) (1 - \epsilon) +  \frac{\epsilon(|A|-1)}{|A|}.
\end{eqnarray}

As expected, when $\beta = 1$, then the first term is zero and we have the same result as for an optimal policy. Also, if $\epsilon = 1$, then all actions are random and we get the same result as for an optimal policy, since the optimality of the cloned policy does not affect the analysis if all actions are random. 

If we assume that $|A|$ is large, then we have $(|A| - 1)/|A| \approx 1$ and we have 
\begin{equation}
p_\epsilon = 1 - \beta (1-\epsilon).
\end{equation}
Thus, the expected number of suboptimal actions in the epsilon greedy cloned policy is 
\begin{equation}
\mathbb{E}[X] = T\cdot(1 - \beta(1-\epsilon)).
\end{equation}
With no noise injection ($\epsilon = 0$), we expect the cloned policy to make $(1-\beta)T$ suboptimal actions. Thus, for large enough $A$ and $T$, as $\epsilon$ goes from 0 to 1 the empirical number of suboptimal actions will, with high probability, interpolate between $(1-\beta)T$ and $T$.

\paragraph{Compounding errors}
The above analysis has ignored the compounding errors that are accumulated when running a behavioral cloned policy in a sequential decision making task. Using the same analysis as Ross et al. \cite{ross2010efficient} we get the following result that proves that the performance of the cloned policy with noise injection $\epsilon$ has an increasingly large performance gap from the expert as $\epsilon$ goes from 0 to 1.

\begin{corollary}
Given a cloned policy $\pi_{BC}$ from expert demonstrations, if the cloned policy makes a mistake with probability $1-\beta$, then 
\begin{equation}
J(\pi_{\rm BC}(\cdot | \epsilon)) \leq J(\pi^*) + T^2 (1 - \beta(1-\epsilon))
\end{equation}
where $T$ is the time-horizon of the MDP.
\end{corollary}
\begin{proof}
As demonstrated above, given $|A|$ sufficiently large, we can define $p_\epsilon$, the probability of taking a suboptimal action when following $\pi_{BC}(\cdot | \epsilon)$ is as follows:
\begin{eqnarray}
p_\epsilon &=& \beta \frac{\epsilon(|A|-1)}{|A|} + (1 - \beta) (1 - \epsilon + \frac{\epsilon(|A|-1)}{|A|})\\
&=& (1-\beta) (1 - \epsilon) +  \frac{\epsilon(|A|-1)}{|A|} \\
&\approx& 1 - \beta(1 - \epsilon).
\end{eqnarray}
The inequality then follows from the proof of Theorem 2.1 in \cite{ross2010efficient}.
\end{proof}

\section{Risk Analysis} \label{app:risk}
Addressing the worst-case performance is important for safe imitation learning \cite{brown2018efficient,brown2018risk}. We investigated the worst-case performance of the demonstrator as it compares to the worst-case performance of the policy learned using D-REX. The results are shown in Table~\ref{tab:drex_risk}. Our results show that D-REX is able to learn safer policies than the demonstrator on 6 out of 7 games via intent extrapolation. The results show that on all games, except for Pong, D-REX is able to find a policy with both higher expected utility (see Table~1 in the main text) as well as higher worst-case utility (Table~\ref{tab:drex_risk}) when compared to the worst case performance of the demonstrator. D-REX also has a better worst-case performance than BC and GAIL across all games except for Pong.

\begin{table*}
\caption{Comparison of the average and worst-case performance of D-REX with respect to the demonstrator. Results are the worst-case performance corresponding to the results shown in Table~1 in the main text. Bold denotes worst-case performance that is better than the worst-case demonstration.}
\label{tab:drex_risk}
\vskip 0.15in
\begin{center}
\begin{small}
\begin{tabular}{ccccc}
\toprule
 & \multicolumn{3}{c}{Worst-Case Performance}\\
 \midrule
Game &  Demonstrator & D-REX & BC & GAIL \\ 
\midrule 
Beam Rider & 900 &\textbf{2916} &	528 & 352 \\
Breakout & 	17 &	\textbf{62} &	13 & 0\\
Enduro &	37	& \textbf{152}	& 35 & 13\\
Pong &	-5 &	-21 &	\textbf{-2} & -14\\
Q*bert &	575 &	\textbf{7675} &	650 & 350\\
Seaquest &	320 &	\textbf{800} &	280 & 260\\
Space Invaders &	190 &	\textbf{575} &	120 & 235\\
\bottomrule
\end{tabular}
\end{small}
\end{center}
\end{table*}

\section{Live-Long Baseline} \label{app:livelong}
Here we describe a simple live-long baseline experiment to test that D-REX is actually learning something other than a positive bonus for living longer. Because we normalize the D-REX learned reward using a sigmoid, one concern is that the non-negativity of the sigmoid is the only thing that is needed to perform well on the Atari domain since games typically involve trying to stay alive as long as possible. We tested this by creating a live-long baseline that always rewards the agent with a +1 reward for every timestep. Table~\ref{tab:drex_baselines} shows that while a +1 reward is sufficient to achieve a moderate score on some games, it is insufficient to learn to play the games Enduro and Seaquest, both of which D-REX is able to learn to play. The reason that a +1 reward does not work on Enduro and Seaquest is that, in these games, it is possible to do nothing and cause an arbitrarily long episode. Thus, simply rewarding long episodes is not sufficient to learn to actually play. While a +1 reward was sufficient to achieve moderate to good scores on the other games, live-long reward is only able to surpass the performance of D-REX on Pong, which gives evidence that even if games where longer trajectories are highly correlated with the ground-truth return, D-REX is not simply rewarding longer episodes, but also rewarding trajectories that follow the demonstrators intention. This is also backed up by our reward attention heat maps shown later, which demonstrate that D-REX is paying attention to details in the observations which are correlated with the ground-truth reward. 

\begin{table*}
\caption{Comparison of D-REX with other imitation learning approaches. BC is behavioral cloning. Live-long assigns every observation a +1 reward and is run using an experimental setup identical to D-REX. }
\label{tab:drex_baselines}
\vskip 0.15in
\begin{center}
\begin{small}
\begin{tabular}{cccccc}
\toprule
 & \multicolumn{2}{c}{Demonstrator} & \multicolumn{3}{c}{Imitation method} \\ 
Game & Avg. & Best. & D-REX & Live-Long & BC\\
\midrule 
Beam Rider  & 1524 & 2216 & \textbf{7220} & 5583.5 & 1268.6\\
Breakout  & 34.5 & 59 & \textbf{94.7} & 68.85 & 29.75\\
Enduro  & 85.5 & 134 & \textbf{247.9} & 0 & 83.4\\
Pong  & 3.7 & \textbf{14} & -9.5 & -5.3 & 8.6\\
Q*bert  & 770 & 850 & \textbf{22543.8} & 17073.75 & 1013.75\\
Seaquest  & 524 & 720 & \textbf{801} & 1 & 530 \\
Space Invaders  & 538.5 & 930 & \textbf{1122.5} & 624 & 426.5 \\
\bottomrule
\end{tabular}
\end{small}
\end{center}
\end{table*}

\section{D-REX Details} \label{app:drex_implementation}
Code and videos can be found at our project website \url{https://dsbrown1331.github.io/CoRL2019-DREX/}.

\subsection{Demonstrations}
To create the demonstrations, we used a partially trained Proximal Policy Optimization (PPO) agent that was checkpointed every 5 optimization steps (corresponds to 10,240 simulation steps) for MuJoCo experiments and 50 optimization steps (corresponds to 51,200 simulation steps) for Atari experiments. To simulate suboptimal demonstrations, we selected demonstration checkpoints such that they resulted in an average performance that was significantly better than random play, but also significantly lower than the maximum performance achieved by PPO when trained to convergence on the ground-truth reward. All checkpoints are included in the source code included in the supplemental materials.

\subsection{Behavioral cloning}

\textbf{MuJoCo experiments} We generated a trajectory of length 1,000, and the given 1,000 pairs of data is used for training. The policy network is optimized with $L_2$ loss for 10,000 iterations using Adam optimizer with a learning rate of 0.001 and a minibatch size of 128. Weight decay regularization is also applied in addition to regular loss term with a coefficient of 0.001. A multi-layer perceptron (MLP) having 4 layers and 256 units in the middle is used to parameterize a policy.

\textbf{Atari experiments} We used the state-action pairs from the 10 demonstrations and partitioned them into an 80\% train 20\% validation split. We used the Nature DQN network architecture and trained the imitation policy using Adam with a learning rate of 0.0001 and a minibatch size of 32. The state consists of four stacked frames which are normalized to have a value between 0 and 1, and the scores in the game scene are masked as it is done in \citet{browngoo2019trex}. We used the validation set for early stopping. In particular, after every 1000 updates on the training data we fully calculated the validation error of the current model. We trained the imitation policy until the validation loss failed to improve for 6 consecutive calculations of the validation error.

\subsection{Synthetic rankings}
We then used the cloned policy and generated 100 synthetic demonstrations for different noise levels. For the MuJoCo experiments, we used 20 different noise levels evenly spaced over the interval $[0.0,1.0)$ and generated 5 trajectories for each level.

For the Atari experiments, we used the noise schedule $\mathcal{E} = (1.0, 0.75, 0.5, 0.25, 0.02)$ and generated $K=20$ trajectories for each level. We found that a non-zero noise was necessary for most Atari games since deterministic policies learned through behavior cloning will often get stuck in a game and fail to take an action to continue playing. For example, in Breakout, it is necessary to release the ball after it falls past the paddle, and a deterministic policy may fail to fire a new ball. For a similar reason, we also found it beneficial to include a few examples of no-op trajectories to encourage the agent to actually complete the game. For each game, we created an additional ``no-op" demonstration set comprised of four length 500 no-op demonstrations. Without these no-op demonstrations, we found that often the learned reward function would give a small positive reward to the agent for just staying alive and sometimes the RL algorithm would decide to just sit at the start screen and accumulate a nearly indefinite stream of small rewards rather than play the game. Adding no-op demonstrations as the least preferred demonstrations shapes the reward function such that it encourages action and progress. While this does encode some amount of domain knowledge into the reward function, it is common that doing nothing is worse than actually attempting to complete a task. We note that in extremely risky scenarios, it may be the case that always taking the no-op action is optimal, but leave these types of domains for future work.

\subsection{Reward function training}

For reward function training, we generally followed the setup used in \cite{browngoo2019trex}. We build a dataset of paired trajectory snippets with ranking, first by choosing two trajectories from given demonstrations and synthetic demonstrations, then by subsampling a snippet from each of trajectory.

\textbf{MuJoCo experiments} We built 3 datasets having different 5,000 pairs and trained a reward function for each of dataset using a neural network. Then, the ensemble of three neural network was used for reinforcement learning step. When two trajectories are selected from synthetic demonstration set to build a dataset, we discarded a pair whose epsilon difference is smaller than 0.3. This stabilizes a reward learning process by eliminating negative samples. Also, when subsampling from a whole trajectory, we limited the maximum length of snippet as 50 while there is no limitation on the minimum length. We then trained each neural network for 1,000 interactions with Adam optimizer with a learning rate of 1e-4 and minibatch size of 64. Weight decay regularization is also used with a coefficient of 0.01. A 3-layer MLP with 256 units in the middle is used to parameterize a reward function.

\textbf{Atari experiments} To generate training samples, we performed data augmentation to generate 40,000 training pairs. We first sampled two noise levels $\epsilon_i$ and $\epsilon_j$, we then randomly sampled one trajectory from each noise level. Finally, we randomly cropped each trajectory keeping between 50 and 200 frames. Following the advice in \cite{browngoo2019trex}, we also enforced a progress constraint such that the randomly cropped snippet from the trajectory with lower noise started at an observation timestep no earlier than the start of the snippet from the higher noise level. To speed up learning, we also only kept every 4th observation. Because observations are stacks of four frames this only removes redundant information from the trajectory. We assigned each trajectory pair a label indicating which trajectory had the lowest noise level.

Given our 40,000 labeled trajectory pairs, we optimized the reward function $\hat{R}_\theta$ using Adam with a learning rate of 1e-5. We held out 20\% of the data as a validation set and optimized the reward function on the training data using the validation data for early stopping. In particular, after every 1000 updates we fully calculated the validation error of the current model. We stopped training once the validation error failed to improve for 6 consecutive calculations of the validation error. 

We used an architecture having four convolutional layers with sizes 7x7, 5x5, 3x3, and 3x3, with strides 3, 2, 1, and 1. The 7x7 convolutional layer used 32 filters and each subsequent convolutional layer used 16 filters and LeakyReLU non-linearities. We then used a fully connected layer with 64 hidden units and a single scalar output. We fed in stacks of 4 frames with pixel values normalized between 0 and 1 and masked reward-related information from the scene; the game score and number of lives, the sector number and number of enemy ships left on Beam Rider, the bottom half of the dashboard for Enduro to mask the position of the car in the race, the number of divers found and the oxygen meter for Seaquest, and the power level and inventory for Hero.

\subsection{Policy optimization}
We optimized a policy by training a PPO agent on the learned reward function. We used the default hyperparameters in OpenAI Baselines.\footnote{https://github.com/openai/baselines} Due to the variability that results from function approximation when using PPO, we trained models using seeds 0, 1, and 2 and reported the best results among them.

\textbf{MuJoCo experiments} We trained an agent for 1 million steps, and gradient is estimated for every 4,096 simulation steps. As same as the original OpenAI implementation, we normalized a reward with running mean and standard deviation. Model ensemble of three neural network is done by averaging such normalized reward.

\textbf{Atari experiments} 9 parallel workers are used to collect trajectories for policy gradient estimation. To reduce reward scaling issues, we followed the procedure proposed by \citet{browngoo2019trex} and normalized predicted rewards by feeding the output of $\hat{R}_\theta(s)$ through a sigmoid function before passing it to PPO. We trained PPO on the learned reward function for 50 million frames to obtain our final policy. 


\section{GAIL}
We used the default implementation of GAIL from OpenAI Baselines for Mujoco. For Atari we made a few changes to get the Baselines implementation to work with raw pixel observations. For the generator policy we used the Nature DQN architecture. The discriminator takes in a state (stack of four frames) and action (represented as a 2-d one-hot vector of shape (84,84,$|\mathcal{A}|$) that is concatenated to the 84x84x4 observation). The architecture for the discriminator is the same as the generator, except that it only outputs two logit values for discriminating between the demonstrations and the generator. We performed one generator update for every discriminator update.

\section{Reward Extrapolation and Attention Heatmaps} \label{app:extrapolate_heat}
Figure~\ref{fig:atari_reward_extrapolation} shows the reward extrapolation plots for all seven Atari games.
Figures~\ref{fig:beamrider_minmax}--\ref{fig:spaceinvaders_minmax} show reward heatmaps for all seven Atari games. We generated the heatmaps using the technique proposed in \cite{browngoo2019trex}. We take a 3x3 mask and run it over every frame in an observation and compute the difference in predicted reward before and after the mask is applied. We then use the cumulative sum over all masks for each pixel to plot the heatmaps. 

\newpage

\begin{figure}
    \centering
    \subfigure[Beam Rider]{
        \includegraphics[width=.31\linewidth]{figs/beamrider_gt_vs_pred_rewards_progress_sigmoid.png}
        
    }
     \subfigure[Breakout]{
        \includegraphics[width=.31\linewidth]{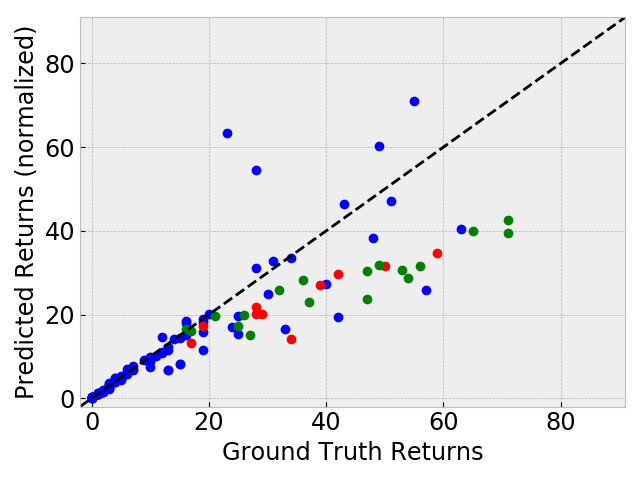}
        
    }
    \subfigure[Enduro]{
        \includegraphics[width=.31\linewidth]{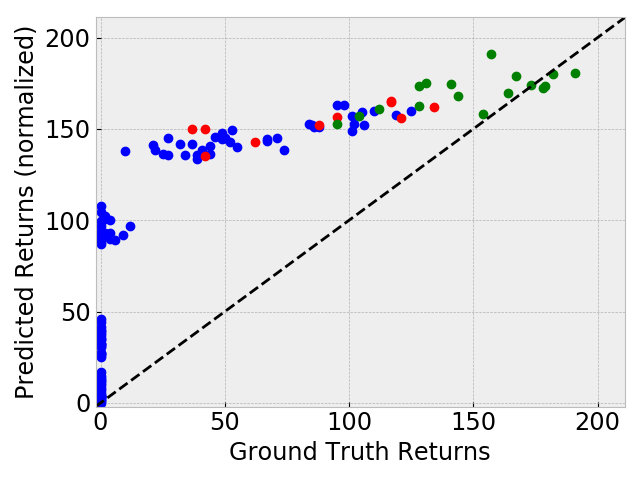}
        
    }
     \subfigure[Pong]{
        \includegraphics[width=.32\linewidth]{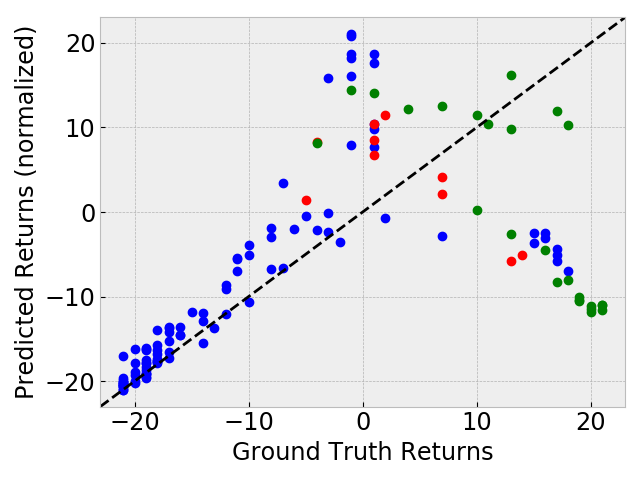}
        
    }
     \subfigure[Q*bert]{
        \includegraphics[width=.32\linewidth]{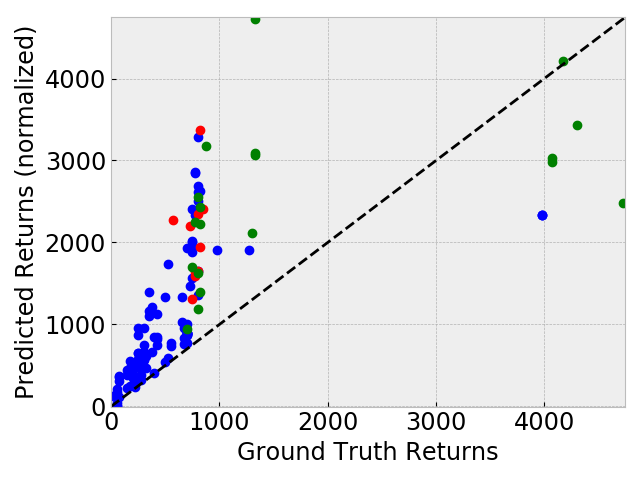}
        
    }
    \subfigure[Seaquest]{
        \includegraphics[width=.32\linewidth]{figs/seaquest_gt_vs_pred_rewards_progress_sigmoid.png}
        
    }
    \subfigure[Space Invaders]{
        \includegraphics[width=.32\linewidth]{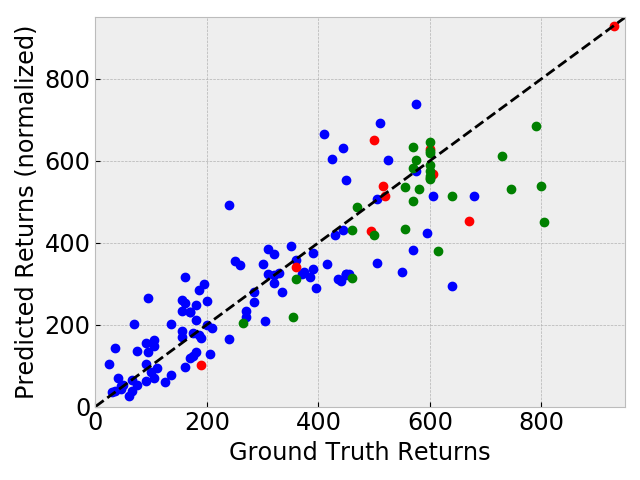}
        
    }
    \caption{Extrapolation plots for Atari games. Blue dots represent synthetic demonstrations generated via behavioral cloning with different amounts of noise injection. Red dots represent actual demonstrations, and green dots represent additional trajectories not seen during training. We compare ground truth returns over demonstrations to the predicted returns using D-REX (normalized to be in the same range as the ground truth returns). }
    \label{fig:atari_reward_extrapolation}
\end{figure}

\newpage

\begin{figure*}
    \centering
    \subfigure[Beam Rider observation with maximum predicted reward]{
        \includegraphics[width=\linewidth]{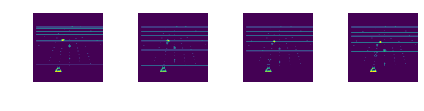}
        
    }
     \subfigure[Beam Rider reward model attention on maximum predicted reward]{
        \includegraphics[width=\linewidth]{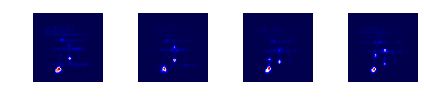}
        
    }
    \subfigure[Beam Rider observation with minimum predicted reward]{
        \includegraphics[width=\linewidth]{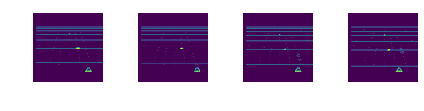}
        
    }
     \subfigure[Beam Rider reward model attention on minimum predicted reward]{
        \includegraphics[width=\linewidth]{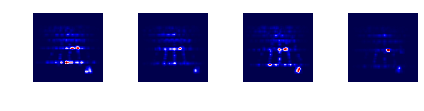}
        
    }
    \caption{Maximum and minimum predicted observations and corresponding attention maps for Beam Rider across a held-out set of 15 demonstrations. The attention maps show that the reward is a function of the status of the controlled ship as well as the enemy ships and missiles.}
    \label{fig:beamrider_minmax}
    
\end{figure*}

\begin{figure*}
    \centering
    \subfigure[Breakout observation with maximum predicted reward]{
        \includegraphics[width=\linewidth]{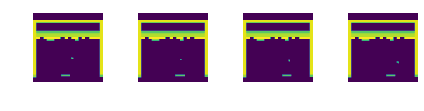}
        \label{subfig:breakout_max}
    }
    \subfigure[Breakout reward model attention on maximum predicted reward]{
        \includegraphics[width=\linewidth]{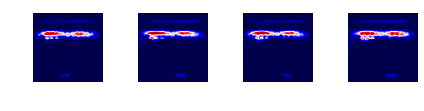}
        \label{subfig:breakout_max_attention}
    }
    \subfigure[Breakout observation with minimum predicted reward]{
        \includegraphics[width=\linewidth]{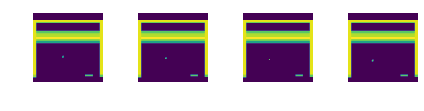}
        \label{subfig:breakout_min}
    }
     \subfigure[Breakout reward model attention on minimum predicted reward]{
        \includegraphics[width=\linewidth]{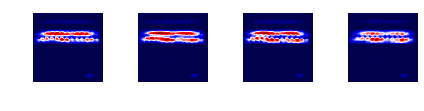}
        \label{subfig:breakout_min_attention}
    }
    \caption{Maximum and minimum predicted observations and corresponding attention maps for Breakout across a held-out set of 15 demonstrations. The observation with maximum predicted reward shows many of the bricks destroyed. The network has learned to put most of the reward weight on the remaining bricks. The observation with minimum predicted reward is an observation where none of the bricks have been destroyed. }
    \label{fig:breakout_minmax}
\end{figure*}

\begin{figure*}
    \centering
    \subfigure[Enduro observation with maximum predicted reward]{
        \includegraphics[width=\linewidth]{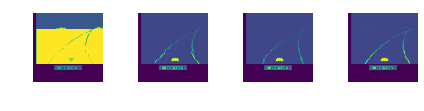}
        
    }
     \subfigure[Enduro reward model attention on maximum predicted reward]{
        \includegraphics[width=\linewidth]{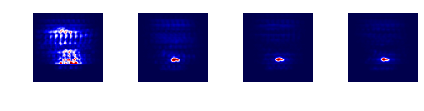}
        
    }
    \subfigure[Enduro observation with minimum predicted reward]{
        \includegraphics[width=\linewidth]{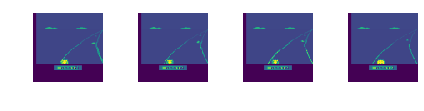}
        
    }
     \subfigure[Enduro reward model attention on minimum predicted reward]{
        \includegraphics[width=\linewidth]{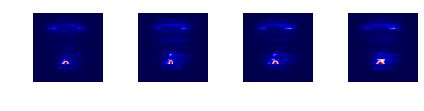}
        
    }
    \caption{Maximum and minimum predicted observations and corresponding attention maps for Enduro across a held-out set of 15 demonstrations. The observation with maximum predicted reward shows the car passing from one section of the race track to another as shown by the change in lighting. The observation with minimum predicted reward shows the controlled car falling behind another racer with attention focusing on the car being controlled as well as the speedometer.}
    \label{fig:enduro_minmax}
\end{figure*}

\begin{figure*}
    \centering
    \subfigure[Pong observation with maximum predicted reward]{
        \includegraphics[width=\linewidth]{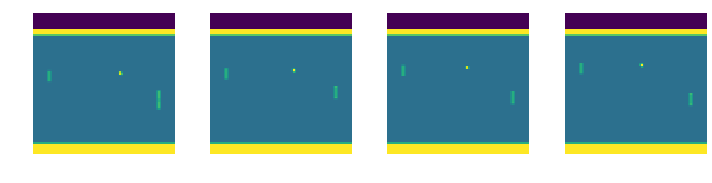}
        
    }
     \subfigure[Pong reward model attention on maximum predicted reward]{
        \includegraphics[width=\linewidth]{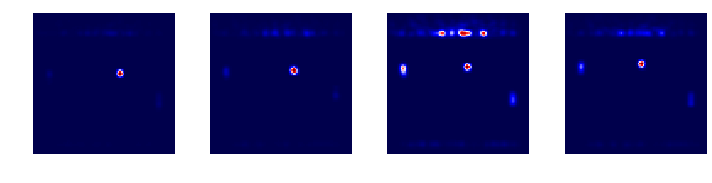}
       
    }
    \subfigure[Pong observation with minimum predicted reward]{
        \includegraphics[width=\linewidth]{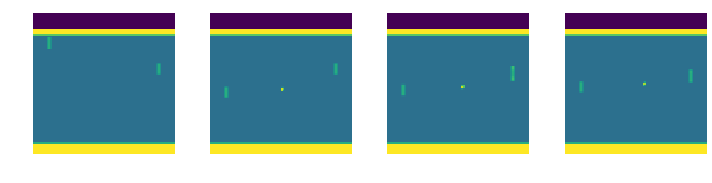}
        
    }
     \subfigure[Pong reward model attention on minimum predicted reward]{
        \includegraphics[width=\linewidth]{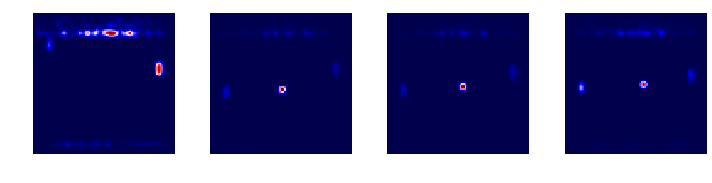}
        
    }
    \caption{Maximum and minimum predicted observations and corresponding attention maps for Pong across  a held-out set of 15 demonstrations. The network attends to the ball and paddles along with some artifacts outside the playing field. The observation with minimum predicted reward shows the ball being sent back into play after the opponent has scored.}
    \label{fig:pong_minmax}
\end{figure*}

\begin{figure*}
    \centering
    \subfigure[Q*bert observation with maximum predicted reward]{
        \includegraphics[width=\linewidth]{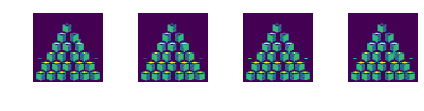}
        
    }
     \subfigure[Q*bert reward model attention on maximum predicted reward]{
        \includegraphics[width=\linewidth]{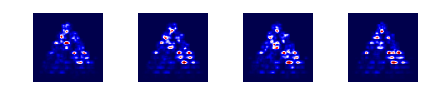}
       
    }
    \subfigure[Q*bert observation with minimum predicted reward]{
        \includegraphics[width=\linewidth]{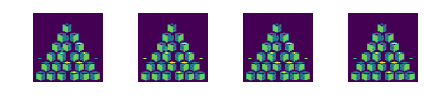}
        
    }
     \subfigure[Q*bert reward model attention on minimum predicted reward]{
        \includegraphics[width=\linewidth]{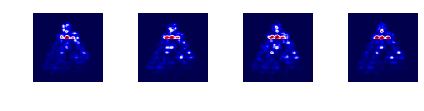}
        
    }
    \caption{Maximum and minimum predicted observations and corresponding attention maps for Q*bert across  a held-out set of 15 demonstrations. The network attention is focused on the different stairs, but is difficult to attribute any semantics to the attention maps.}
    \label{fig:qbert_minmax}
\end{figure*}

\begin{figure*}
    \centering
    \subfigure[Seaquest observation with maximum predicted reward]{
        \includegraphics[width=\linewidth]{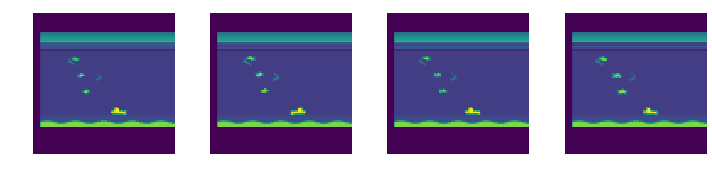}
        
    }
     \subfigure[Seaquest reward model attention on maximum predicted reward]{
        \includegraphics[width=\linewidth]{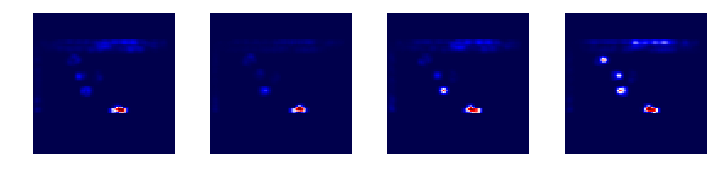}
       
    }
    \subfigure[Seaquest observation with minimum predicted reward]{
        \includegraphics[width=\linewidth]{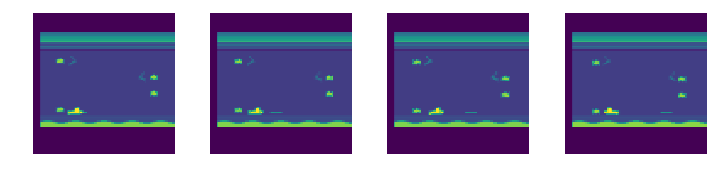}
        
    }
     \subfigure[Seaquest reward model attention on minimum predicted reward]{
        \includegraphics[width=\linewidth]{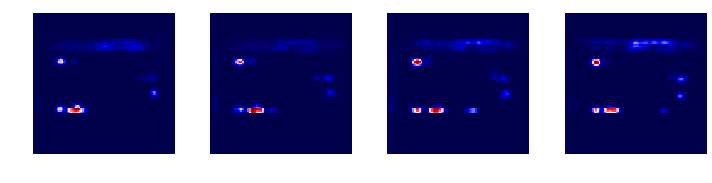}
        
    }
    \caption{Maximum and minimum predicted observations and corresponding attention maps for Seaquest across a held-out set of 15 demonstrations. The observation with maximum predicted reward shows the submarine in a safe location with no immediate threats. The observation with minimum predicted reward shows the submarine one frame before it is hit and destroyed by an enemy shark. This is an example of how the network has learned a shaped reward that helps it play the game better than the demonstrator. The network has learned to give most attention to nearby enemies and to the controlled submarine. }
    \label{fig:seaquest_minmax}
\end{figure*}

\begin{figure*}
    \centering
    \subfigure[Space Invaders observation with maximum predicted reward]{
        \includegraphics[width=\linewidth]{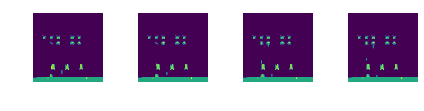}
        
    }
     \subfigure[Space Invaders reward model attention on maximum predicted reward]{
        \includegraphics[width=\linewidth]{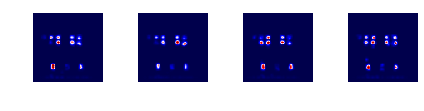}
       
    }
    \subfigure[Space Invaders observation with minimum predicted reward]{
        \includegraphics[width=\linewidth]{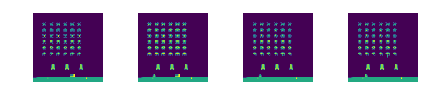}
        
    }
     \subfigure[Space Invaders reward model attention on minimum predicted reward]{
        \includegraphics[width=\linewidth]{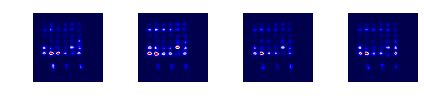}
        
    }
    \caption{Maximum and minimum predicted observations and corresponding attention maps for Space Invaders across a held-out set of 15 demonstrations. The observation with maximum predicted reward shows an observation where most of the all the aliens have been successfully destroyed and the protective barriers are still intact. The attention map shows that the learned reward function is focused on the barriers and aliens, with less attention to the location of the controlled ship. The observation with minimum predicted reward shows the very start of a game with all aliens still alive. The network attends to the aliens and barriers, with higher weight on the aliens and barrier closest to the space ship.}
    \label{fig:spaceinvaders_minmax}
\end{figure*}



\end{document}